\documentclass[sn-basic]{sn-jnl}

\usepackage{graphicx}%
\usepackage{multirow}%
\usepackage{amsmath,amssymb,amsfonts}%
\usepackage{amsthm}%
\usepackage{mathrsfs}%
\usepackage[title]{appendix}%
\usepackage{xcolor}%
\usepackage{textcomp}%
\usepackage{manyfoot}%
\usepackage{booktabs}%
\usepackage{algorithmicx}%
\usepackage{algpseudocode}%
\usepackage{listings}%

\usepackage[ruled,linesnumbered,noend]{algorithm2e} 
\usepackage{subcaption} 
\usepackage{hyperref}
\usepackage{bm}
\usepackage{bbm}
\usepackage{enumitem}

\theoremstyle{thmstyleone}%
\newtheorem{theorem}{Theorem}

\theoremstyle{thmstyletwo}%

\theoremstyle{thmstylethree}%

\usepackage[normalem]{ulem} 
\newcounter{gaocomm}  
\definecolor{blue-violet}{rgb}{0.00,0.75,0.90}  
\definecolor{bostonuniversityred}{rgb}{1.0, 0.0, 0.0}

\begin{document}

\newcommand{\rv}[1]{\textcolor{black}{#1}}
\newcommand{\rvtwo}[1]{\textcolor{black}{#1}}
\newcommand{\qj}[1]{\textcolor{black}{#1}}
\newcommand{\jy}[1]{\textcolor{red}{#1}}

\newcommand{\qjd}[1]{\textcolor{orange}{#1}}

\title[Combating Confirmation Bias: A Unified Pseudo-Labeling Framework for Entity Alignment]{Combating Confirmation Bias: A Unified Pseudo-Labeling Framework for Entity Alignment}


\author[1]{\fnm{Qijie} \sur{Ding}}\email{qijie.ding@sydney.edu.au}

\author*[1]{\fnm{Jie} \sur{Yin}}\email{jie.yin@sydney.edu.au}

\author[2]{\fnm{Daokun} \sur{Zhang}}\email{daokun.zhang@monash.edu}

\author[1]{\fnm{Junbin} \sur{Gao}}\email{junbin.gao@sydney.edu.au}

\affil[1]{\orgdiv{Discipline of Business Analytics}, \orgname{The University of Sydney}, \orgaddress{
\city{Sydney}, 
\state{NSW}, \country{Australia}}}

\affil[2]{\orgdiv{Department of Data Science and Artificial Intelligence}, \orgname{Monash University}, \orgaddress{\city{Melbourne}, \state{VIC}, \country{Australia}}}


\abstract{Entity alignment (EA) aims at identifying equivalent entity pairs across different knowledge graphs (KGs) that refer to the same real-world identity. It has been a compelling but challenging task that requires the integration of heterogeneous information from different KGs to expand the knowledge coverage and enhance inference abilities. To circumvent the shortage of seed alignments provided for training, recent EA models utilize pseudo-labeling strategies to iteratively add unaligned entity pairs predicted with high confidence to the seed alignments for model training. However, the adverse impact of confirmation bias during pseudo-labeling has been largely overlooked, thus hindering entity alignment performance. To systematically combat confirmation bias, we propose a new \underline{U}nified \underline{P}seudo-\underline{L}abeling framework for \underline{E}ntity \underline{A}lignment (UPL-EA) that explicitly alleviates pseudo-labeling errors to boost the performance of entity alignment. \rv{UPL-EA achieves this goal through two key innovations:} (1) Optimal Transport (OT)-based pseudo-labeling uses discrete OT modeling as an effective means to determine entity correspondences and reduce erroneous matches across two KGs. An effective criterion is derived to infer pseudo-labeled alignments that satisfy one-to-one correspondences; (2) \rvtwo{Parallel pseudo-label ensembling refines pseudo-labeled alignments by combining predictions over multiple models independently trained in parallel.} 
The ensembled pseudo-labeled alignments are thereafter used to augment seed alignments to reinforce subsequent model training for alignment inference. The effectiveness of UPL-EA in eliminating pseudo-labeling errors is both theoretically supported and experimentally validated. \rv{Our extensive results and in-depth analyses demonstrate the superiority of UPL-EA over 15 competitive baselines and its utility as a general pseudo-labeling framework for entity alignment.}}




\maketitle

\section{Introduction}

Knowledge Graphs (KGs) are large-scale structured knowledge bases that represent real-world entities (or concepts) and their relationships as a collection of factual triplets. Recent years have witnessed the release of various open-source KGs (e.g., Freebase~\citep{freebase}, YAGO~\citep{yago} and Wikidata~\citep{wikidata}) from general to specific domains and their proliferation to empower many artificial intelligence (AI) applications, such as recommender systems~\citep{guo2020survey}, question answering~\citep{yang2018hotpotqa} and information retrieval~\citep{paulheim2017knowledge}. Nevertheless, it has become a well-known fact that real-world KGs suffer from incompleteness arising from their complex, semi-automatic construction process. This has led to an increasing number of research efforts on KG completion, such as TransE~\citep{bordes2013translating} and TransH~\citep{wang2014knowledge}, which aim to add missing facts to individual KGs.
Unfortunately, due to its limited coverage and incompleteness, a single KG cannot fulfill the requirements for complex AI applications that build upon heterogeneous knowledge sources. This necessitates the integration of heterogeneous information from multiple individual KGs to enrich knowledge representation. Entity alignment (EA) is a crucial task towards this objective, which aims to establish the correspondence between equivalent entity pairs across different KGs that refer to the same real-world identity. 

Over the last decade, there has been a surge of research efforts dedicated to entity alignment across KGs. Most mainstream EA models are embedding-based; they embed KGs into a shared latent embedding space so that similarities between entities can be measured via their embeddings for alignment inference. 
\rv{To leverage structural information in KGs, more recent EA models exploit the power of Graph Neural Networks (GNNs) to encode KG structures for entity alignment.
Methods like GCN-Align~\citep{wang2018cross} utilize GCNs to learn better entity embeddings by aggregating features from neighboring entities. However, GCNs and their variants suffer from an over-smoothing issue~\citep{min2020scattering, jiang2022sparse}, where the embeddings of entities among local neighborhoods become indistinguishably similar as the number of convolution layers increases. 
To alleviate over-smoothing during GCN neighborhood aggregation,
recent works~\cite{wu2019jointly,wu2019relation,zhu2021relation2} use a highway strategy~\citep{srivastava2015highway} on GCN layers, which ``mixes" the smoothed entity embeddings with the original features. Despite achieving competitive results, these methods require an abundant amount of pre-aligned entity pairs (known as \textit{seed alignments}) provided for training}, which are labor-intensive and costly to acquire in real-world KGs. To tackle the shortage of seed alignments, recently proposed models, such as BootEA~\citep{sun2018bootstrapping}, IPTransE~\citep{zhu2017iterative}, MRAEA~\citep{mao2020mraea}, and RNM~\citep{zhu2021relation2}, adopt a bootstrapping strategy that iteratively selects unaligned entity pairs predicted with high confidence as pseudo-labeled alignments and adds them to seed alignments for subsequent model training. 
The bootstrapping strategy, originating from the field of statistics, is also referred to as pseudo-labeling---a predominant learning paradigm proposed to tackle label scarcity in semi-supervised learning. 


In general semi-supervised learning, pseudo-labeling approaches inherently suffer from confirmation bias~\citep{arazo2020pseudo,Tarvainen2017mean}. \rv{The confirmation bias refers to using incorrectly predicted labels generated on unlabeled data for subsequent training, thereby misleadingly increasing model confidence in incorrect predictions and leading to a biased model with degraded performance. } Unfortunately, there is a lack of understanding of the fundamental factors that give rise to confirmation bias for pseudo-labeling-based entity alignment. Our analysis (see Section~\ref{subsec:analysis}) advocates that the confirmation bias is exacerbated during pseudo-labeling for entity alignment. Due to the lack of sufficient seed alignments at the early stages of training, \rv{the existing models tend to learn uninformative entity embeddings and consequently generate error-prone pseudo-labeled alignments based on unreliable model predictions.}
We characterize pseudo-labeling errors into two types: (1) \rvtwo{\textbf{Conflicted misalignments}, where a single entity in one KG is simultaneously aligned with more than one entity in another KG with erroneous matches, violating the one-to-one correspondence}. 
(2) \rvtwo{\textbf{One-to-one misalignments}, where an entity in one KG is aligned to a single but incorrect counterpart in another KG}. The pseudo-labeling errors, if not properly mitigated, would propagate into subsequent model training, thereby jeopardizing the efficacy of pseudo-labeling-based entity alignment. However, current pseudo-labeling-based EA models have made only limited attempts to alleviate alignment conflicts using simple heuristics~\citep{zhu2017iterative,sun2019transedge,mao2020mraea,zhu2021relation2} or imposing constraints to enforce hard alignments~\citep{sun2018bootstrapping,ding2022conflict}, while the confirmation bias has been left under-explored.


\rv{To address the research gap, we propose a novel Unified Pseudo-Labeling framework for Entity Alignment (UPL-EA) aimed at alleviating confirmation bias and improving entity alignment performance. The key idea lies in ``reliably" pseudo-labeling unaligned entity pairs based on model predictions and augmenting seed alignments to iteratively improve model performance. UPL-EA comprises two essential components: Optimal Transport (OT)-based pseudo-labeling and \rvtwo{parallel pseudo-label ensembling}, to effectively reduce pseudo-labeling errors. OT-based pseudo-labeling considers entity alignment as a probabilistic matching process between entity sets in two KGs.}
\rv{An effective criterion is mathematically derived to select pseudo-labeled alignments that satisfy one-to-one correspondences, thus mitigating conflicted misalignments in model predictions. 
\rvtwo{Parallel pseudo-label ensembling reduces variability in pseudo-label selection by deriving  consensus predictions from multiple OT-based models trained in parallel}, thereby mitigating one-to-one misalignments.} The ensembled pseudo-labeled alignments are then used to augment seed alignments to reinforce subsequent model training for alignment inference. \rv{To our best knowledge, we are the first to address the confirmation bias inherent in pseudo-labeling-based entity alignment. Comprehensive experiments and analyses validate the superior performance of UPL-EA over state-of-the-art supervised and semi-supervised baselines and its utility as a general pseudo-labeling framework to improve entity alignment performance.}

The remainder of this paper is organized as follows. Section~\ref{section: preliminaries} provides a problem statement of pseudo-labeling-based entity alignment and presents an empirical analysis of confirmation bias that motivates this work. Section~\ref{section: method} presents the proposed framework, followed by an in-depth experimental evaluation reported in Section~\ref{section: experiments}. Related works are discussed in Section~\ref{section:related work}, and we conclude the paper in Section~\ref{section:conclusion}.

\section{Preliminaries}
\label{section: preliminaries}

In this section, we first provide a problem statement of pseudo-labeling-based entity alignment. Then, we perform a thorough analysis of confirmation bias during pseudo-labeling, which motivates the design of our proposed framework.

\subsection{Problem Statement}

A knowledge graph (KG) can be represented as $\mathcal{G} = \{\mathcal{E}, \mathcal{R}, \mathcal{T}\}$ with the entity set $\mathcal{E}$, relation set $\mathcal{R}$, and relational triplet set $\mathcal{T}$. \rv{Each triplet is denoted as $(e_{i}, r, e_{j}) \in \mathcal{T}$, which represents that a head entity $e_i\in\mathcal{E}$ is connected to a tail entity $e_j\in\mathcal{E}$ via a relation $r\in\mathcal{R}$. Each entity $e_i$ is characterized by an entity feature vector $ \bm{x}_i\in \mathbb{R}^{n}$, which can be obtained from entity names with semantic meanings.}

\rv{Formally, given two KGs, $\mathcal{G}_{1} = \{\mathcal{E}_{1}, \mathcal{R}_{1}, \mathcal{T}_{1}\}$ and $\mathcal{G}_{2} = \{\mathcal{E}_{2},\mathcal{R}_{2}, \mathcal{T}_{2}\}$, the task of entity alignment (EA) aims to discover a set of one-to-one equivalent entity pairs $\mathcal{I} = \{(e_i,e_j)\in\mathcal{E}_1\times\mathcal{E}_2|\;e_i\equiv e_j\}$ between $\mathcal{G}_1$ and $\mathcal{G}_2$, where $e_i \in \mathcal{E}_1$, $e_j \in \mathcal{E}_2$, and $\equiv$ indicates an equivalence relationship between $e_i$ and $e_j$. In many cases, a small set of equivalent entity pairs \rvtwo{$\mathcal{S} \subset \mathcal{I}$, known as \textit{prior seed alignments}, is provided beforehand and used for training.} Apart from the entities included in $\mathcal{S}$, there are two sets of unaligned entities $\mathcal{E}^{U}_1\subset \mathcal{E}_1$ and $\mathcal{E}^{U}_2 \subset \mathcal{E}_2$ in $\mathcal{G}_1$ and $\mathcal{G}_2$, respectively. 
}

\rv{In this work, we address real-world scenarios where seed alignments $\mathcal{S}$ are often scarce due to high labeling costs. We focus on the task of pseudo-labeling-based entity alignment, which aims to leverage both seed alignments and unaligned entity pairs to more effectively train an EA model in \rvtwo{a transductive semi-supervised setting}. This is achieved by selecting a set of unaligned entity pairs as pseudo-labeled alignments $\mathcal{S}_t\subset\mathcal{E}^{U}_1\times\mathcal{E}^{U}_2$ in each iteration $t$, and using $\mathcal{S}_t$ to iteratively augment seed alignments $\mathcal{S}$, i.e., $\mathcal{S}\leftarrow\mathcal{S}\cup \mathcal{S}_t$, for subsequent model training.}

\subsection{Analysis of Confirmation Bias}
\label{subsec:analysis}

The key to pseudo-labeling-based entity alignment lies in selecting reliable pseudo-labeled alignments to effectively boost model performance; otherwise, pseudo-labeling errors could propagate into subsequent model training, leading to confirmation bias~\citep{arazo2020pseudo}. 
To investigate the impact of confirmation bias on pseudo-labeling-based entity alignment, we perform an error analysis of a naive pseudo-labeling strategy used in previous studies~\citep{sun2019transedge}. This strategy simply selects pairs of unaligned entities whose embedding distances (defined using Eq.~(\ref{eq:dist})) are smaller than a pre-specified threshold as pseudo-labeled alignments. Our analysis is carried out on a widely used cross-lingual KG pair, DBP15K\textsubscript{ZH\_EN}, as detailed in Section~\ref{subsec: datasets and baselines}. \rvtwo{We follow the conventional 30\%-70\% split ratio to randomly partition 15,000 ground-truth alignments into training and test data. During model training, pseudo-labeled alignments are inferred from unaligned entity sets to augment seed alignments.} To understand the underlying causes of confirmation bias, we explicitly calculate, in each pseudo-labeling iteration, the numbers of \rvtwo{conflicted misalignments and one-to-one misalignments}, as well as the number of \rvtwo{correct one-to-one alignments, against ground-truth alignments} on the test data. 

\rvtwo{As shown in Fig.~\ref{fig:naive error number}, the naive strategy for selecting pseudo-labeled alignments introduces a substantial number of pseudo-labeling errors---including both conflicted misalignments and one-to-one misalignments---from the beginning. As training progresses, although the number of correct one-to-one alignments gradually increases owing to improved entity embeddings, both types of pseudo-labeling errors also accumulate and increase noticeably. This accumulation of errors gives rise to confirmation bias, severely hindering the capability of pseudo-labeling in improving entity alignment performance, as shown in Fig.~\ref{fig:accuracy}.}


\begin{figure*}[t]
     \begin{subfigure}[b]{0.32\textwidth}
         \centering
\includegraphics[width=\textwidth]{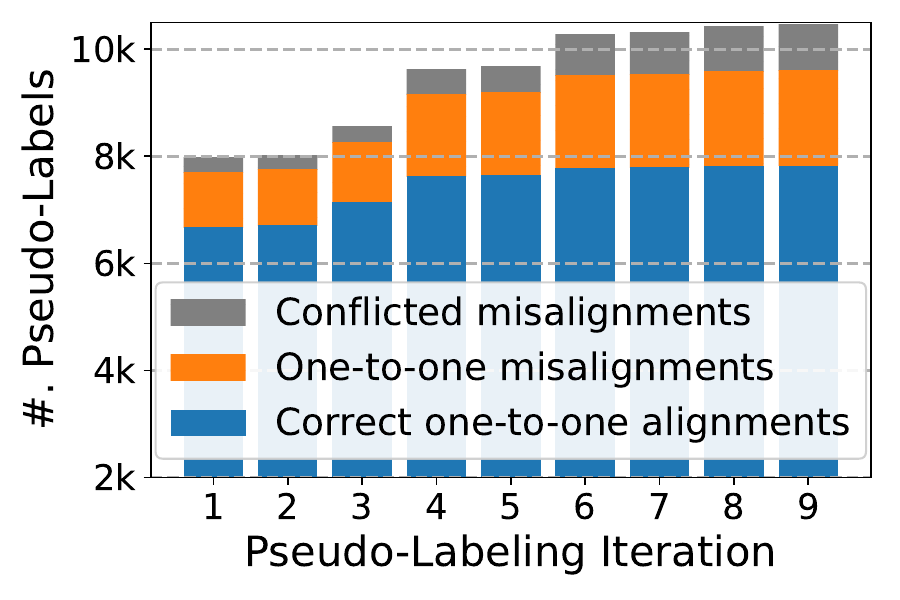}
\caption{\rvtwo{The naive pseudo-labeling strategy}}
\label{fig:naive error number}
     \end{subfigure}
    \hfill
    \begin{subfigure}[b]{0.33\textwidth}
        \centering
        \includegraphics[width=\textwidth]{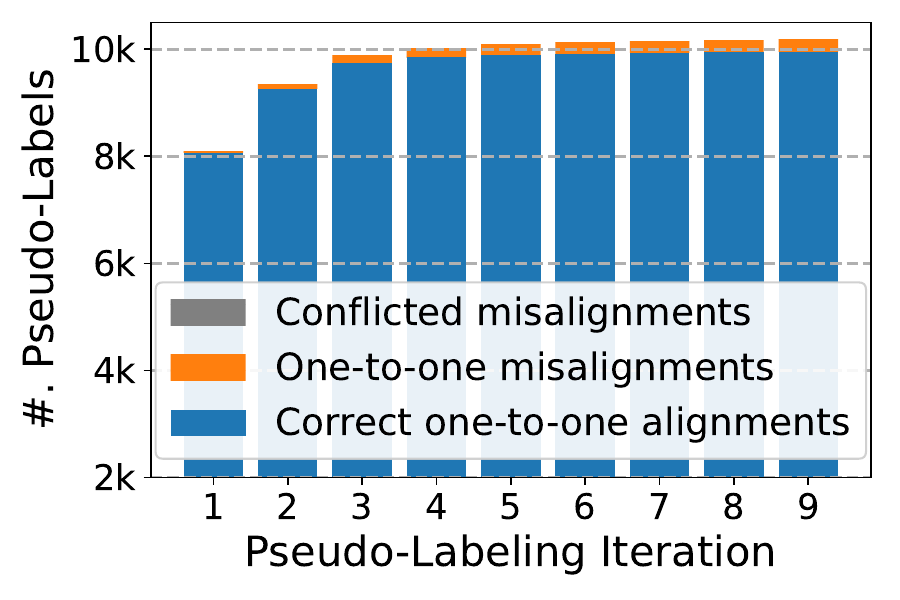}
        \caption{\rvtwo{Our unified pseudo-labeling (UPL) strategy}}
        \label{fig:uplea error number}
    \end{subfigure}
     \hfill
     \begin{subfigure}[b]{0.33\textwidth}
         \centering        \includegraphics[width=\textwidth]{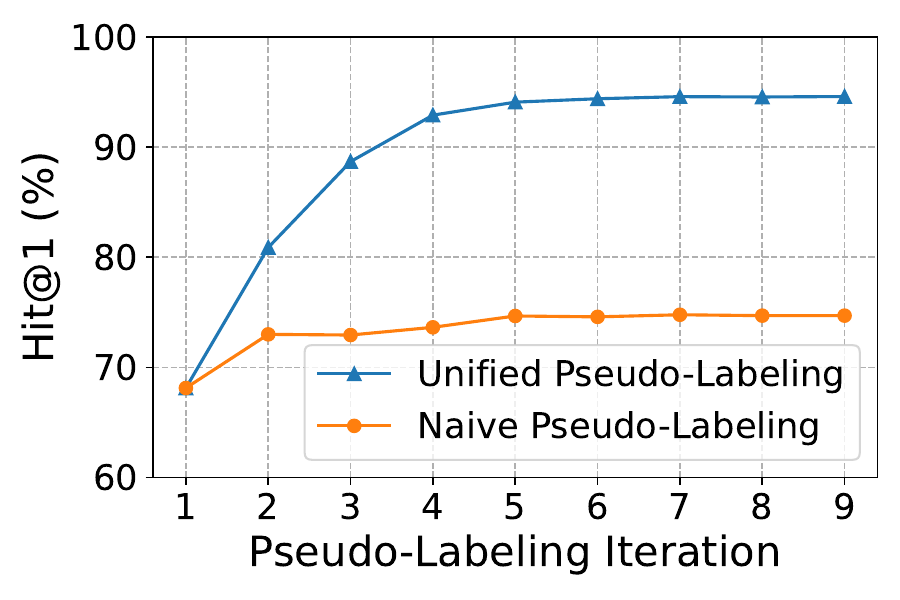}        
         \caption{\rvtwo{Entity alignment performance of two strategies}}
         \label{fig:accuracy}
     \end{subfigure}
\caption{Error analysis of confirmation bias. (a) Number of pseudo-labeled alignments selected by a naive pseudo-labeling strategy over pseudo-labeling iterations. (b) Number of pseudo-labeled alignments selected by the proposed UPL strategy over pseudo-labeling iterations. (c) Entity alignment performance comparison (Hit@1) of the naive pseudo-labeling strategy and the proposed UPL strategy. }
\label{fig:PL errors rate}
\end{figure*}

Our analysis affirms that the confirmation bias essentially stems from pseudo-labeling errors. These errors, if not adequately addressed, would propagate through subsequent model training and jeopardize the effectiveness of pseudo-labeling for entity alignment. \rvtwo{Motivated by this, our work focuses on explicitly identifying and eliminating the two types of pseudo-labeling errors: conflicted misalignments and one-to-one misalignments. As shown in Fig.~\ref{fig:uplea error number}, our proposed Unified Pseudo-Labeling (UPL) strategy can effectively eliminate all conflicted misalignments through OT-based pseudo-labeling and significantly reduce one-to-one misalignments via parallel pseudo-label ensembling across all iterations. Fig.~\ref{fig:accuracy} highlights the significant performance gains achieved by our proposed UPL strategy compared to the naive pseudo-labeling strategy.}



\section{The Proposed UPL-EA Framework}
\label{section: method}


With insights from our analysis in Section~\ref{subsec:analysis}, the proposed UPL-EA framework is designed to systematically address confirmation bias for pseudo-labeling-based entity alignment. The core of UPL-EA is a novel Unified Pseudo-Labeling (UPL) strategy that iteratively generates reliable pseudo-labeled alignments to enhance the \rvtwo{training of} entity alignment (EA) model. Essentially, UPL utilizes two key components: (1) OT-based pseudo-labeling, which generates pseudo-labeled alignments with one-to-one correspondences, effectively \rvtwo{eliminating conflicted misalignments; and (2) parallel pseudo-label ensembling, which combines pseudo-labeled alignments generated from multiple OT-based models independently trained in parallel, reducing variability in pseudo-label selection and mitigating one-to-one misalignments.}

\begin{figure*}[h]
\centering
\includegraphics[width=\textwidth]{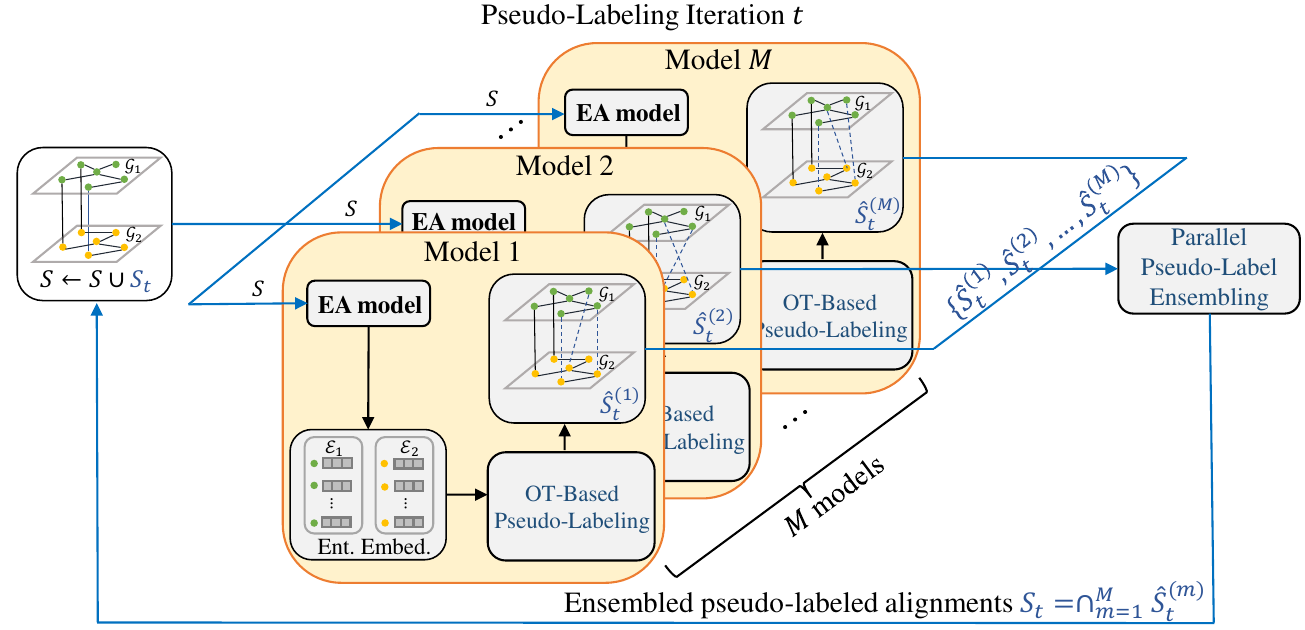}\vspace{0.2cm}
    \caption{\rvtwo{An overview of the proposed UPL-EA framework. In pseudo-labeling iteration $t$, $M$ EA models are trained in parallel to generate a set of conflict-free pseudo-labeled alignments via OT-based pseudo-labeling. These alignments are further fed into parallel pseudo-label ensembling to generate ensembled pseudo-labeled alignments, which are then used to augment seed alignments for subsequent model training in the next iteration $t+1$. }} 
    \label{fig:Framework_Overview}
\end{figure*}

\rvtwo{Fig.~\ref{fig:Framework_Overview} illustrates an overview of the proposed UPL-EA framework. In pseudo-labeling iteration $t$, an EA model learns entity embeddings based on a set of alignment seeds, $\mathcal{S}$, which are passed on to OT-based pseudo-labeling to generate conflict-free pseudo-labeled alignments. Rather than relying on predictions from a single model that are potentially unreliable, $M$ EA models are independently trained in parallel to generate a set of conflict-free pseudo-labeled alignments $\{\widehat{\mathcal{S}}_t^{(1)},\widehat{\mathcal{S}}_t^{(2)},...,\widehat{\mathcal{S}}_t^{(M)}\}$ through OT-based pseudo-labeling. These alignments are then combined through parallel pseudo-label ensembling, which retains only the consensus pseudo-labeled alignments generated from $M$ models, i.e., $\mathcal{S}_t=\cap_{m=1}^{M}\widehat{\mathcal{S}}_t^{(m)}$. The ensembled pseudo-labeled alignments $\mathcal{S}_t$ are then used to augment seed alignments, i.e., $\mathcal{S}\leftarrow\mathcal{S}\cup \mathcal{S}_t$, for subsequent model training in the next pseudo-labeling iteration $t+1$. Through this iterative process, the EA models and the UPL strategy mutually reinforce each other, progressively leading to more informative entity embeddings. Finally, the learned entity embeddings are used for entity alignment inference.}



\subsection{Entity Alignment Model}
\label{section:EAmodel}

The entity alignment (EA) model aims to learn informative entity embeddings and perform model training for entity alignment inference. 

\rvtwo{To enable the EA model to learn informative entity embeddings, our UPL-EA framework adopts a modular entity embedding encoder, $f_{\textrm{en}}$, designed to effectively capture structural information inherent in KGs. Given an entity $e_i \in \mathcal{E}_1 \cup \mathcal{E}_2$ with its associated feature vector $\bm{x}_i\in\mathbb{R}^n$, the encoder $f_{\textrm{en}}:\mathbb{R}^n \rightarrow \mathbb{R}^d $ maps the entity into a $d$-dimensional embedding space $\mathbb{R}^d$:
\begin{equation}
\label{eq:ent_embed}
\bm{h}_i=f_{\textrm{en}}(\bm{x}_i,\mathcal{G}_1, \mathcal{G}_2;\bm{\Theta}),
\end{equation}
where $\mathcal{G}_1$ and $\mathcal{G}_2$ are the two KGs, $\bm{\Theta}$ represents the learnable parameters of the EA model, and $\bm{h}_i\in\mathbb{R}^d$ is the encoded entity embedding. 
The encoder $f_{\textrm{en}}$ can be instantiated with any expressive entity embedding model. In this work, we adopt a highway-gated GCN with a global-local neighborhood aggregation scheme, following our previous work~\citep{ding2022conflict}.}

After obtaining the encoded entity embeddings $\{\bm{h}_i|\;e_i\in \mathcal{E}_1\cup\mathcal{E}_2\}$, \rv{we then define a margin-based loss function for embedding learning, such that the equivalent entities are encouraged to be close to each other in the embedding space: }
\rv{\begin{equation}
L=\sum_{(e_{i}, e_{j}) \in \mathcal{S}}\sum_{(e_i^-, e_j^-) \in\mathcal{S}^-_{(e_i,e_j)}} \max\bigl(0, d(e_i, e_j)-d(e_i^-, e_j^-)+\gamma\bigl),
\label{eq:hard_loss}
\end{equation}
where $\mathcal{S}$ denotes a set of \rvtwo{prior} seed alignments initially provided for training. \rv{$\mathcal{S}^-_{(e_i,e_j)}=\{(e_i,e_j^-)|\;e_j^-\in\mathcal{E}_2\setminus\{e_j\}\}\cup\{(e_i^-,e_j)|\;e_i^-\in\mathcal{E}_1\setminus\{e_i\}\}$} denotes the set of negative alignments, synthesized by negative sampling of a positive alignment $(e_i, e_j) \in \mathcal{S}$ as $(e_i, e_j^-) \notin \mathcal{S}$ and $(e_i^-, e_j) \notin \mathcal{S}$.} $\gamma$ is a hyper-parameter that determines the margin that separates positive alignments from negative alignments. $d(e_i,e_j)$ indicates the embedding distance between entity pair $(e_i,e_j)$ across two KGs, defined as:
\begin{equation}
  d(e_i,e_j)=\|\bm{h}_{i}-\bm{h}_{j}\|_1.
\label{eq:dist}
\end{equation}

\rv{The loss function in Eq.~(\ref{eq:hard_loss}) can be minimized with respect to entity embeddings $\{\bm{h}_i|\;e_i\in \mathcal{E}_1\cup\mathcal{E}_2\}$. To facilitate model training,} we adopt an \textit{adaptive negative sampling} strategy to obtain a set of negative alignments $\mathcal{S}^-_{(e_i,e_j)}$. Specifically, for each positive alignment $(e_i,e_j) \in \mathcal{S}$, we select $K$ nearest entities of $e_i$ (or $e_j$), measured using the embedding distance in Eq.~(\ref{eq:dist}), to replace $e_j$ (or $e_i$) and form $K$ negative counterparts $(e_i,e_j^-)$ (or $(e_i^-,e_j)$). This strategy helps generate ``hard" negative alignments and pushes their associated entities to be apart from each other in the embedding space. 

\subsection{\rv{Unified Pseudo-Labeling Strategy}}
\rv{After training the EA model, the learned entity embeddings can be used for alignment inference. However, as seed alignments initially provided for training are often limited, the performance of entity alignment can be suboptimal. Therefore, pseudo-labeling strategies can be designed to select a set of unaligned entity pairs as pseudo-labeled alignments, which are used to augment seed alignments for boosting model training. However, as previously discussed in Section~\ref{subsec:analysis}, a naive strategy inevitably introduces a considerable number of pseudo-labeling errors, leading to confirmation bias. To address this issue, we propose UPL, a unified pseudo-labeling framework that explicitly aims to mitigate \rvtwo{conflicted misalignments and one-to-one misalignments}.}

\rv{In what follows, the two key components of UPL: OT-based pseudo-labeling and parallel pseudo-label ensembling, are discussed in detail.
}



\subsubsection{OT-based Pseudo-Labeling}
\label{subsec:OT}

\rv{To mitigate \rvtwo{conflicted misalignments}, we propose using Optimal Transport (OT) as an effective means to reliably pseudo-label unaligned entity pairs across KGs to reinforce the training of the EA model. As a powerful mathematical framework for transforming one distribution into another, OT allows to more effectively identify one-to-one correspondences between entities, ensuring a coherent alignment configuration.}

Formally, we model the alignment between two unaligned entity sets $\mathcal{E}^{U}_1$ and $\mathcal{E}^{U}_2$ as an OT process to warrant the one-to-one correspondence, i.e., transporting each entity $e_i \in \mathcal{E}^{U}_1$ to a unique entity $e_j \in \mathcal{E}^{U}_2$, with minimal overall transport cost. Denote $\bm{C}\in \mathbb{R}^{|\mathcal{E}^{U}_1|\times |\mathcal{E}^{U}_2|}$ as the transport cost matrix, and without loss of generality, we assume $|\mathcal{E}^{U}_1|<|\mathcal{E}^{U}_2|$. The transport plan is a mapping function $T: e_i \rightarrow T(e_j)$, where $e_i\in\mathcal{E}^{U}_1$, $T(e_i)\in\mathcal{E}^{U}_2$. Thus, the objective of entity alignment is to find the optimal transport plan $T^*$ that minimizes the overall transport cost: 
\begin{equation}
T^*=\mathop{\arg\min}\limits_{T} \sum_{e_i\in \mathcal{E}^{U}_1} C_{e_i, T(e_i)}.
\label{eq:general_ot_problem}
\end{equation}

\rv{A critical aspect of the above objective is defining a reliable measure of the transport cost. One might directly use the distances between the learned entity embeddings. However, 
when only a limited number of seed alignments are available for training, the learned entity embeddings can be uninformative, particularly during the early training stages before the EA model has converged. As a result, using these embeddings to calculate the distances for defining the transport cost can be error-prone.} To address this, we resort to rectifying the embedding distance using relational neighborhood matching~\citep{zhu2021relation2}\rv{, which complements the training of the EA model, especially during the early stages, for learning better entity embeddings and providing a more reliable cost measure for OT modeling.} The principle of distance rectification is to leverage relational contexts within local neighborhoods to help determine the extent to which two entities should be aligned. Intuitively, if two entities $e_i \in \mathcal{E}_1$ and $e_j \in \mathcal{E}_2$ share more aligned neighboring entities/relations, the distance between their embeddings should be smaller, \rv{indicating a higher likelihood of being aligned to each other. Based on this intuition, the transport cost for OT-based pseudo-labeling is defined as follows:}
\begin{equation}
\label{eq:transport cost}
C_{e_i,e_j} =\rv{d(e_i,e_j)-\lambda s(e_i,e_j)},\; e_i\in \mathcal{E}^{U}_1, e_j\in \mathcal{E}^{U}_2,
\end{equation}
where $\lambda$ is a trade-off hyper-parameter, and $s(e_i, e_j)$ is a scoring function indicating the degree to which the relational contexts of two entities $e_i$ and $e_j$ match. 
\rv{Let $\mathcal{M}_{(e_i,e_j)}$ represent the set of aligned neighboring relation-entity tuples for entity pair $(e_i, e_j)$, obtained following~\citep{zhu2021relation2}. The score $s(e_i, e_j)$ is calculated as:
\begin{equation}
    s(e_i, e_j)= \frac{\sum_{(r, e_k)\in\mathcal{M}_{(e_i,e_j)}}\xi_1(r, e_k)\xi_2(r, e_k)}{|\mathcal{N}_e(e_{i})|+|\mathcal{N}_e({e}_{j})|},
\end{equation}
where $\mathcal{N}_e(e_i)$ and $\mathcal{N}_e(e_j)$ denote the sets of neighboring entities for $e_i$ and $e_j$, respectively. $\xi_1(r, e_k)$ and $\xi_2(r, e_k)$ indicate the reciprocal frequency of triplets associated with neighboring tuple $(r, e_k)$ for $e_i$ in $\mathcal{T}_1$ and $e_j$ in $\mathcal{T}_2$, respectively. }


\rv{The objective in Eq.~(\ref{eq:general_ot_problem}) defines a hard assignment optimization problem, which, however, does not scale well}. To enable more efficient optimization and to allow for a more flexible alignment configuration, we reformulate this objective as a discrete OT problem, where the optimal transport plan is considered as a coupling matrix $\bm{P}^*\in \mathbb{R}_{+}^{|\mathcal{E}^{U}_1|\times|\mathcal{E}^{U}_2|}$ between two discrete distributions. Denote $\mu$ and $\nu$ as two discrete probability distributions over all entities $\{e_i|e_i\in \mathcal{E}^{U}_1\}$ and $\{e_j|e_j\in \mathcal{E}^{U}_2\}$, respectively. Without any alignment preference, the two discrete distributions $\mu$ and $\nu$ are assumed to follow a uniform distribution such that $\mu=\frac{1}{|\mathcal{E}^{U}_1|}\sum_{e_i\in \mathcal{E}^{U}_1}\delta_{e_i}$ and $\nu=\frac{1}{|\mathcal{E}^{U}_2|}\sum_{e_j\in \mathcal{E}^{U}_2}\delta_{e_j}$, where $\delta_{e_i}$ and $\delta_{e_j}$ are the Dirac function centered on $e_i$ and $e_j$, respectively. Both $\mu$ and $\nu$ are bounded to sum up to one:
 $\sum_{e_i\in \mathcal{E}^{U}_1}\mu(e_i)=\sum_{e_i\in \mathcal{E}^{U}_1}\frac{1}{|\mathcal{E}^{U}_1|}=1$ and $\sum_{e_j\in \mathcal{E}^{U}_2}\nu(e_j)=\sum_{e_j\in \mathcal{E}^{U}_2}\frac{1}{|\mathcal{E}^{U}_2|}=1$. \rv{Accordingly, the OT objective is formulated to find the optimal coupling matrix $\bm{P}^*$ between $\mu$ and $\nu$:}
\begin{align}
\label{eq:specific_ot_problem}
\bm{P}^*=&\mathop{\arg\min}\limits_{\bm{P}\in \Pi(\mu, \nu)} \sum_{e_i\in \mathcal{E}^{U}_1} \sum_{e_j\in \mathcal{E}^{U}_2}P_{e_i,e_j}\cdot C_{e_i,e_j}, \\
\text{ subject to:} &\sum_{e_j\in \mathcal{E}^{U}_2} P_{e_i,e_j}=\mu(e_i)=\frac{1}{|\mathcal{E}^{U}_1|},\nonumber\\
& \sum_{e_i\in \mathcal{E}^{U}_1} P_{e_i,e_j}=\nu(e_j)=\frac{1}{|\mathcal{E}^{U}_2|},\nonumber\\
& P_{e_i,e_j}\geq0, \forall e_i\in \mathcal{E}^{U}_1, \forall e_j\in \mathcal{E}^{U}_2, \nonumber
\end{align}
where $\Pi(\mu, \nu)=\{\bm{P}\in\mathbb{R}_{+}^{|\mathcal{E}^{U}_1|\times|\mathcal{E}^{U}_2|}|\;\bm{P}\mathbf{1}_{|\mathcal{E}^{U}_2|}=\mu,\;\bm{P}^\top \mathbf{1}_{|\mathcal{E}^{U}_1|}=\nu\}$ is the set of all joint probability distributions with marginal probabilities $\mu$ and $\nu$, $\mathbf{1}_n$ denotes an $n$-dimensional vector of ones. $\bm{P}$ is a coupling matrix signifying probabilistic alignments between two unaligned entity sets $\mathcal{E}^{U}_1$ and $\mathcal{E}^{U}_2$. Therefore, $P_{e_i,e_j}$ indicates the amount of probability mass transported from $\mu(e_i)$ to $\nu(e_j)$. A larger value of $P_{e_i,e_j}$ indicates a higher likelihood of $e_i$ and $e_j$ being aligned to each other. 

To solve the discrete OT problem in Eq.~(\ref{eq:specific_ot_problem}), several exact algorithms have been proposed, such as interior point methods~\citep{wachter2006implementation} and network simplex~\citep{orlin1997polynomial}. \rv{While these exact algorithms guarantee to find a closed-form optimal transport plan, their high computational cost makes them intractable for iterative pseudo-labeling. 
Thus}, we propose to use \rv{an} entropy regularized OT problem, as defined in Eq.~(\ref{eq:sinkhorn_ot_problem}) below, which can be solved by the efficient Sinkhorn algorithm~\citep{cuturi2013sinkhorn}:
\begin{equation}
\begin{aligned}
\bm{P}^*=\mathop{\arg\min}\limits_{P\in \Pi(\mu, \nu)} \sum_{e_i\in \mathcal{E}^{U}_1} \sum_{e_j\in \mathcal{E}^{U}_2}P_{e_i,e_j}\cdot C_{e_i,e_j} + \beta \sum_{e_i\in \mathcal{E}^{U}_1} \sum_{e_j\in \mathcal{E}^{U}_2} P_{e_i,e_j} \text{log}P_{e_i,e_j},
\label{eq:sinkhorn_ot_problem}
\end{aligned}
\end{equation}
where $\beta$ is a hyper-parameter that controls the \rv{strength} of regularization. Solving the above entropy regularized OT problem can be easily implemented using popular deep-learning frameworks \rv{such as} PyTorch and TensorFlow. 


Once the optimal coupling matrix $\bm{P}^*$ is estimated, entity alignments can be inferred accordingly. Since one-to-one correspondences are crucial for eliminating conflicted misalignments, \rv{we further propose a selection criterion to identify entity pairs as pseudo-labeled alignments:}
\begin{equation}
\rv{\widehat{\mathcal{S}}_t} = \{(e_i,e_j)|\;P^*_{e_i, e_j}>\frac{1}{2\cdot\text{min}(|\mathcal{E}^{U}_1|, |\mathcal{E}^{U}_2|)}, \ e_i\in \mathcal{E}^{U}_1, e_j \in \mathcal{E}^{U}_2\}. 
\end{equation}
This criterion ensures that the selected \rv{pseudo-labeled alignments} satisfy one-to-one correspondence with theoretical guarantees. \rv{Unlike in previous works~\citep{zhu2017iterative,sun2019transedge,mao2020mraea,zhu2021relation2,ding2022conflict}, this approach does not require pre-specifying the threshold}.  

\vspace{0.3cm}
\begin{theorem}[]\label{theorem}
\rv{Any pseudo-labeled alignment $(e_i,e_j), e_i\in \mathcal{E}^{U}_1, e_j \in \mathcal{E}^{U}_2$ that satisfies the condition $P^*_{e_i, e_j}>\frac{1}{2\cdot \text{min}(|\mathcal{E}^{U}_1|, |\mathcal{E}^{U}_2|)}$ warrants one-to-one correspondence, such that no conflicted entity pairs, $\{(e_i,e_k)|\;e_k\in \mathcal{E}^{U}_2 \setminus \{e_j\}\}$ and $\{(e_l,e_j)|\;e_l\in \mathcal{E}^{U}_1 \setminus \{e_i\}\}$, are selected as pseudo-labeled alignments.}
\end{theorem}

\begin{proof}
Given the optimized coupling matrix $\bm{P}^*\in \mathbb{R}_{+}^{|\mathcal{E}^{U}_1|\times|\mathcal{E}^{U}_2|}$ \rv{subject to} the constraints of  $\sum_{e_j\in \mathcal{E}^{U}_2} P^*_{e_i,e_j}=\frac{1}{|\mathcal{E}^{U}_1|}$ for all rows ($\forall e_i\in \mathcal{E}^{U}_1$) and $\sum_{e_i\in \mathcal{E}^{U}_1} P^*_{e_i,e_j}=\frac{1}{|\mathcal{E}^{U}_2|}$ for all columns ($\forall e_j\in \mathcal{E}^{U}_2$). Assume \rv{that} $|\mathcal{E}^{U}_1|<|\mathcal{E}^{U}_2|$, the decision threshold is $\frac{1}{2\cdot\text{min}(|\mathcal{E}^{U}_1|, |\mathcal{E}^{U}_2|)}=\frac{1}{2|\mathcal{E}^{U}_1|}$. Entity pairs $\{(e_i,e_j)|\;P^*_{e_i, e_j}>\frac{1}{2|\mathcal{E}^{U}_1|}, e_i\in \mathcal{E}^{U}_1, e_j \in \mathcal{E}^{U}_2\}$ are selected as \rv{pseudo-labeled alignments}.

For each pseudo-labeled \rv{alignment} $(e_i,e_j)$ with a probability value $P^*_{e_i,e_j}>\frac{1}{2|\mathcal{E}^{U}_1|}$, 
we can prove that no conflicted \rv{entity pairs} $\{(e_i,e_k)|\;e_k\in \mathcal{E}^{U}_2 \setminus \{e_j\}\}$ associated with $e_i$ are selected as \rv{pseudo-labeled alignments}:
\begin{align}
\label{eq:rearrange}
P^*_{e_i,e_j}&>\frac{1}{2|\mathcal{E}^{U}_1|},\nonumber\\
\sum_{e_j\in \mathcal{E}^{U}_2}P^*_{e_i,e_j}-P^*_{e_i,e_j}&<\sum_{e_j\in \mathcal{E}^{U}_2}P^*_{e_i,e_j}-\frac{1}{2|\mathcal{E}^{U}_1|},\nonumber\\
\sum_{e_k\in \mathcal{E}^{U}_2 \setminus \{e_j\}}P^*_{e_i,e_k}+P^*_{e_i,e_j}-P^*_{e_i,e_j}&<\frac{1}{|\mathcal{E}^{U}_1|}-\frac{1}{2|\mathcal{E}^{U}_1|},\nonumber\\
\sum_{e_k\in \mathcal{E}^{U}_2 \setminus \{e_j\}}P^*_{e_i,e_k}&<\frac{1}{2|\mathcal{E}^{U}_1|}.
\end{align}

Since the coupling matrix $\bm{P}^*\in \mathbb{R}_{+}^{|\mathcal{E}^{U}_1|\times|\mathcal{E}^{U}_2|}$ has non-negative entries, the summation $\sum_{e_k\in \mathcal{E}^{U}_2 \setminus \{e_j\}}P^*_{e_i,e_k}$ from Eq.~(\ref{eq:rearrange}) must be no smaller than \rvtwo{any of its components}, i.e., $P^*_{e_i,e_k}\leq\sum_{e_k\in \mathcal{E}^{U}_2 \setminus \{e_j\}}P^*_{e_i,e_k}$, $\forall e_k\in \mathcal{E}^{U}_2 \setminus \{e_j\}$. \rvtwo{Please note that, as we focus on a transductive semi-supervised setting, the size of two unaligned entity sets, $|\mathcal{E}^{U}_1|$ and $|\mathcal{E}^{U}_2|$, are known.} Together with Eq.~(\ref{eq:rearrange}), we can further derive that any component in the summation is smaller than the decision threshold, i.e., $P^*_{e_i,e_k}\rv{\leq\sum_{e_k\in \mathcal{E}^{U}_2 \setminus \{e_j\}}P^*_{e_i,e_k}}<\frac{1}{2|\mathcal{E}^{U}_1|}$, $\forall e_k\in \mathcal{E}^{U}_2 \setminus \{e_j\}$. 
In other words, all other probability values in the same row of $P^*_{e_i,e_j}$ are smaller than the decision threshold. Thus, no conflicted entity pairs $\{(e_i,e_k)|\;e_k\in \mathcal{E}^{U}_2 \setminus \{e_j\}\}$ associated with $e_i$ are selected as \rv{pseudo-labeled alignments}. 

Similarly, for each pseudo-labeled \rv{alignment} $(e_i,e_j)$ with a probability value $P^*_{e_i,e_j}>\frac{1}{2|\mathcal{E}^{U}_1|}$, we can prove that no conflicted entity pairs $\{(e_l,e_j)|\;e_l\in \mathcal{E}^{U}_1 \setminus \{e_i\}\}$ associated with entity $e_j$ are selected as \rv{pseudo-labeled alignments}. Similar to Eq.~(\ref{eq:rearrange}), we can also obtain $\sum_{e_l\in \mathcal{E}^{U}_1 \setminus \{e_i\}}P^*_{e_l,e_j}<\frac{1}{2|\mathcal{E}^{U}_2|}$ and $P^*_{e_l,e_j}\leq \sum_{e_l\in \mathcal{E}^{U}_1 \setminus \{e_i\}}P^*_{e_l,e_j},\forall e_l\in \mathcal{E}^{U}_1 \setminus \{e_i\}$. Together with the assumption of $|\mathcal{E}^{U}_1|<|\mathcal{E}^{U}_2|$, we can further derive that $P^*_{e_l,e_j}\leq \sum_{e_l\in \mathcal{E}^{U}_1 \setminus \{e_i\}}P^*_{e_l,e_j}<\frac{1}{2|\mathcal{E}^{U}_2|}<\frac{1}{2|\mathcal{E}^{U}_1|}$. Therefore, all other probability values in the same column of $P^*_{e_i,e_j}$ are smaller than the decision threshold, i.e., $P^*_{e_l,e_j}<\frac{1}{2|\mathcal{E}^{U}_1|}$, $\forall e_l\in \mathcal{E}^{U}_1 \setminus \{e_i\}$. Hence, no conflicted \rv{entity pairs} $\{(e_l,e_j)|\;e_l\in \mathcal{E}^{U}_1 \setminus \{e_i\}\}$ associated with entity $e_j$ are selected as \rv{pseudo-labeled alignments}. 

In summary, we conclude that the selected \rv{pseudo-labeled alignments} $\{(e_i,e_j)|\;P^*_{e_i, e_j}>\frac{1}{2\cdot \text{min}(|\mathcal{E}^{U}_1|,|\mathcal{E}^{U}_2|)}, e_i\in \mathcal{E}^{U}_1, e_j \in \mathcal{E}^{U}_2\}$ are guaranteed to be one-to-one alignments when $|\mathcal{E}^{U}_1|<|\mathcal{E}^{U}_2|$, and the same conclusion also holds when $|\mathcal{E}^{U}_1|\geq|\mathcal{E}^{U}_2|$. 
\end{proof}

\rv{The OT-based pseudo-labeling algorithm} is provided in Algorithm~\ref{algorithm_sinkhorn}. \rv{In Step 1, the algorithm starts by calculating the transport cost matrix $\bm{C}$, with a time complexity of $O(|\mathcal{E}^{U}_1|\cdot|\mathcal{E}^{U}_2|\cdot d)$, where $d$ is the embedding dimension}. In Steps 2-9, the Sinkhorn algorithm takes the transport cost matrix $\bm{C}$ as input to estimate the optimal transport plan $\bm{P}^*$ via iterative row normalization and column normalization, the time complexity is $O(|\mathcal{E}^{U}_1|\cdot|\mathcal{E}^{U}_2|/\beta)$. In Step 10, entity pairs with values in $\bm{P}^*$ larger than the decision threshold are selected as \rv{pseudo-labeled alignments}. Finally, in Step 11, the algorithm returns a set of \rv{conflict-free pseudo-labeled alignments $\widehat{\mathcal{S}}_t$}. The overall time complexity of Algorithm~\ref{algorithm_sinkhorn} is $O(|\mathcal{E}^{U}_1|\cdot|\mathcal{E}^{U}_2|\cdot d)$. 

\begin{algorithm}[h!]
\caption{OT-based Pseudo-Labeling with Sinkhorn Algorithm \label{algorithm_sinkhorn}}
\KwIn{\rv{Unaligned entity sets $\mathcal{E}^{U}_1 \subset \mathcal{E}_1$ and $\mathcal{E}^{U}_2 \subset \mathcal{E}_2$, entity embeddings $\{\bm{h}_i|\; e_i\in\mathcal{E}^{U}_1 \cup \mathcal{E}^{U}_2\}$ and
regularization hyper-parameter $\beta$.}}
\KwOut{Pseudo-labeled \rv{alignments $\widehat{\mathcal{S}}_t$}} 
\rv{
Calculate transport cost $\bm{C}$ according to Eq.~(\ref{eq:transport cost})}\;
Initialize $\bm{P}=\mathbf{1}_{|\mathcal{E}^{U}_1|}\mathbf{1}_{|\mathcal{E}^{U}_2|}^\top$\; 
$\bm{a}=\frac{1}{|\mathcal{E}^{U}_1|}\mathbf{1}_{|\mathcal{E}^{U}_1|}$, $\bm{Z}=\text{e}^{-\frac{1}{\beta}\bm{C}}$\;
\Repeat{\rv{convergence or reaching a fixed number of iterations}}{
$\bm{Q}=\bm{Z}\odot \bm{P}$\;
$\bm{b} = \frac{1}{|\mathcal{E}^{U}_1|Q^\top \bm{a}}$, $\bm{a} = \frac{1}{|\mathcal{E}^{U}_2|Q \bm{b}}$\;
$\bm{P} = \bm{a}\bm{b}^\top \odot \bm{Q}$\;
}
Obtain the optimal transport plan $\bm{P}^*=\bm{P}$\;
\rv{Select} conflict-free \rv{pseudo-labeled alignments} $\rv{\widehat{\mathcal{S}}_t} = \{(e_i,e_j)|\;P^*_{e_i, e_j}>\frac{1}{2\cdot\text{min}(|\mathcal{E}^{U}_1|, |\mathcal{E}^{U}_2|)}, e_i\in \mathcal{E}^{U}_1, e_j \in \mathcal{E}^{U}_2\}$\; 
\Return{pseudo-labeled \rv{alignments $\widehat{\mathcal{S}}_t$}.}
\end{algorithm}

\subsubsection{\rvtwo{Parallel Pseudo-Label Ensembling}}

\rvtwo{Through OT-based pseudo-labeling, conflict-free pseudo-labeled alignments are inferred. However, these alignments remain susceptible to one-to-one misalignments, particularly when the learned entity embeddings are still uninformative in the early stages of model training. To further mitigate one-to-one misalignments, ensemble learning can be exploited to reduce variability in pseudo-label selection in semi-supervised settings. Self-ensembling methods, such as temporal ensembling~\citep{laine2017temporal}, aggregate the predictions from a single model across different training epochs. Although self-ensembling has been shown to enhance prediction consistency in semi-supervised settings, it can inadvertently introduce confirmation bias by imposing cross-iteration dependencies and exacerbating error propagation in the context of pseudo-labeling. }

\rvtwo{To address this issue, we propose a parallel ensembling approach that refines pseudo-labeled alignments by combining predictions from multiple OT-based models trained in parallel. By seeking consensus across multiple models, this approach reduces prediction variability, thus enhancing the overall quality and reliability of pseudo-labeled alignments. }

\rvtwo{Formally, in pseudo-labeling iteration $t$, given $M$ sets of pseudo-labeled alignments $\{\widehat{\mathcal{S}}_t^{(1)},\widehat{\mathcal{S}}_t^{(2)},...,\widehat{\mathcal{S}}_t^{(M)}\}$ inferred from $M$ OT-based models independently trained in parallel, we generate the ensembled pseudo-labeled alignments by taking those that are consistently selected across all sets in $\{\widehat{\mathcal{S}}_t^{(1)},\widehat{\mathcal{S}}_t^{(2)},...,\widehat{\mathcal{S}}_t^{(M)}\}$:}
\begin{equation}
\rvtwo{\mathcal{S}_t =\cap_{m=1}^{M}\widehat{\mathcal{S}}_t^{(m)}= \bigl\{(e_i,e_j)|\sum_{m=1}^{M}\mathbbm{1}\bigl( (e_i,e_j)\in\widehat{\mathcal{S}}_t^{(m)}\bigl) =M, e_i\in \mathcal{E}^{U}_1, e_j\in \mathcal{E}^{U}_2\bigl\},}
\end{equation}
where $\mathbbm{1}(\cdot)$ is a binary indicator function. To decorrelate the dependency among $M$ EA models, their model parameters are initialized independently. 

By focusing on consensus alignments from multiple OT-based models, our parallel ensembling approach is expected to achieve higher pseudo-labeling precision compared to relying on a single model’s predictions. \rvtwo{This strategy relates to consistency-based techniques such as Dual Student~\citep{ke2019dual}, which combines loosely coupled predictions from two independently trained models and introduces a stabilization constraint in the loss function to enhance prediction consistency on unlabeled data. While sharing a similar objective, our method offers a simpler yet effective alternative, as empirically demonstrated in Section~\ref{subsection:ensembling}.}

Finally, the ensembled pseudo-labeled alignments \rvtwo{$\mathcal{S}_t$} are used to augment seed alignments $\mathcal{S}$ as follows:
\begin{equation}
\rvtwo{\mathcal{S}\leftarrow\mathcal{S} \cup \mathcal{S}_t}.
\end{equation}
The augmented seed alignments $\mathcal{S}$ include a considerable number of reliable pseudo-labeled alignments, which in turn strengthen subsequent model training.

\begin{algorithm}[t]
\caption{UPL-EA Training Process}\label{algorithm1}
\KwIn{Two KGs $\mathcal{G}_1=\{\mathcal{E}_1,\mathcal{R}_1,\mathcal{T}_1\}$, $\mathcal{G}_2=\{\mathcal{E}_2,\mathcal{R}_2,\mathcal{T}_2\}$, seed alignments $\mathcal{S}$.}
\KwOut{Learned entity embeddings.}
\Repeat{convergence or reaching a fixed number of iterations}{
\ForPar{\rvtwo{$m$ in $1$ \dots $M$}}{
\rvtwo{Learn entity embeddings based on $\mathcal{S}$ by minimizing Eq.~(\ref{eq:hard_loss})\;}
Infer pseudo-labeled alignments \rvtwo{$\widehat{\mathcal{S}}_t^{(m)}$} via Algorithm~\ref{algorithm_sinkhorn}\; 
}
Ensemble pseudo-labeled alignments \rvtwo{$\mathcal{S}_t=\cap_{m=1}^{M}\widehat{\mathcal{S}}_t^{(m)}$}\;
Augment seed alignments: \rvtwo{$\mathcal{S}\leftarrow\mathcal{S} \cup \mathcal{S}_t$};
}
\Return{learned entity embeddings.} 
\end{algorithm}

\subsection{Overall Training Procedure} 

\rvtwo{The UPL-EA training procedure is summarized in Algorithm~\ref{algorithm1}. 
In Steps 2-4, $M$ EA models are trained in parallel using seed alignments to obtain informative entity embeddings, which are then passed on to OT-based pseudo-labeling to infer conflict-free pseudo-labeled alignments.
In Steps 5-6, parallel pseudo-label ensembling is performed over $M$ models to generate ensembled pseudo-labeled alignments, which are thereafter used to augment seed alignments for subsequent model training. }
Finally, the learned entity embeddings are returned as the output. For alignment inference, the learned entity embeddings are used to infer new aligned entity pairs via Algorithm~\ref{algorithm_sinkhorn}.

\rv{In Algorithm~\ref{algorithm1}, the time complexity of learning entity embeddings in Step 3 is $O((|\mathcal{E}_1\cup\mathcal{E}_2|)\cdot d^2)$, where $d$ is the entity embedding dimension, and the time complexity of OT-based pseudo-labeling in Step 4 is $O(|\mathcal{E}^{U}_1|\cdot|\mathcal{E}^{U}_2|\cdot d)$. Thus, the time complexity of Algorithm~\ref{algorithm1} is $O(I\cdot(|\mathcal{E}^{U}_1| \cdot |\mathcal{E}^{U}_2|\cdot d + (|\mathcal{E}_1\cup\mathcal{E}_2|) \cdot d^2))$, where $I$ is the maximum number of pseudo-labeling iterations. }

\section{Experiments}
\label{section: experiments}

In this section, we validate the efficacy of our proposed UPL-EA framework through extensive experiments, ablation studies and in-depth analyses on benchmark datasets. 

\subsection{\rv{Experimental Settings}}
\label{subsec: datasets and baselines}
This section presents the experimental settings, including benchmark datasets used and detailed experimental setup.

\subsubsection{Datasets}
To evaluate the effectiveness of our UPL-EA framework, we carry out experiments on both cross-lingual datasets \rv{and cross-source monolingual datasets}. The statistics of all datasets are summarized in Table~\ref{tab:dataset}.
\begin{table}[h!]
   \caption{Statistics of benchmark datasets}
   \label{tab:dataset}
\begin{tabular}{c|c|c c c}
      \toprule 
	  \multicolumn{2}{c|}{\textbf{Datasets}} & \textbf{Entities} & \textbf{Relations} & \textbf{Rel.triplets} \\
      \midrule 
	  \multirow{2}{*}{DBP15K\textsubscript{ZH\_EN}} & Chinese & 66,469 & 2,830 & 153,929 \\
                                                      & English & 98,125 & 2,317 & 237,674 \\
      \midrule 
	  \multirow{2}{*}{DBP15K\textsubscript{JA\_EN}} & Japanese & 65,744 & 2,043 & 164,373 \\
                                                      & English & 95,680 & 2,096 & 233,319 \\
      \midrule 
	  \multirow{2}{*}{DBP15K\textsubscript{FR\_EN}} & French & 66,858 & 1,379 & 192,191 \\
                                                      & English & 105,889 & 2,209 & 278,590 \\
      \midrule 
	  \multirow{2}{*}{SRPRS\textsubscript{EN\_FR}} & English & 15,000 & 221 & 36,508 \\
                                                     & French & 15,000 & 177 & 33,532 \\
      \midrule 
	  \multirow{2}{*}{SRPRS\textsubscript{EN\_DE}} & English & 15,000 & 222 & 38,363 \\
                                                     & German &  15,000 & 120 & 37,377 \\
      \midrule 
	  \multirow{2}{*}{\rv{DBP-YG-15K (OpenEA)}} & \rv{English} & \rv{15,000} & \rv{165} & \rv{30,291} \\
                                            & \rv{English} & \rv{15,000} & \rv{28} & \rv{26,638} \\
      \midrule 
	\multirow{2}{*}{\rv{DBP-YG-15K (RealEA)}} & \rv{English} & \rv{19,865} & \rv{290} & \rv{60,329} \\
                                            & \rv{English} & \rv{21,050} & \rv{32} & \rv{82,109} \\
      \bottomrule 
    \end{tabular}
\end{table}


\noindent\textbf{Cross-Lingual Datasets.} DBP15K~\citep{sun2017cross} is a widely used benchmark dataset for cross-lingual entity alignment~\citep{ding2022conflict, liu2022selfkg, wu2019relation, zhu2021relation2}. It includes three cross-lingual KG pairs extracted from DBpedia, each containing two KGs built upon English and another language (Chinese, Japanese, or French), with 15,000 aligned entity pairs per dataset. SRPRS~\citep{guo2019learning} is a more recent benchmark dataset characterized by sparser connections~\citep{guo2019learning} that are extracted from DBpedia. SRPRS comprises two cross-lingual KG pairs, each with two KGs in English and French/German, and it also includes 15,000 aligned entity pairs.

\vspace{0.2cm}
\noindent\rv{\textbf{Cross-Source Monolingual Datasets.}
DBP-YG-15K is a cross-source monolingual dataset extracted from DBpedia~\citep{auer2007dbpedia} and YAGO 3~\citep{yago}. DBP-YG-15K has two KG pairs, sampled by OpenEA~\citep{sun2020benchmarking} and RealEA~\citep{leone2022critical}, respectively. The OpenEA KG pair is constructed without duplicated entities in each KG, which aligns with our approach. However, the RealEA KG pair does not use this setting and allow duplicated entities in one KG. Both KG pairs of DBP-YG-15K are built upon English and each has 15,000 aligned entity pairs.}

\subsubsection{Experimental Setup}
For fair comparisons, we follow the conventional 30\%-70\% split ratio to randomly partition training and test data on all datasets. 
We use semantic meanings of entity names to construct entity features. On DBP15K with relatively larger linguistic barriers, we first use Google Translate to translate non-English entity names into English, then look up 768-dimensional word embeddings pre-trained by BERT~\citep{devlin2018bert} with English entity names to form entity features. On SRPRS \rv{and DPB-YG-15K}, 
we directly look up word embeddings without translation. As each entity name comprises one or multiple words, we further use TF-IDF to measure the contribution of each word towards entity name representation. Finally, we aggregate TF-IDF-weighted word embeddings for each entity to form its entity feature vector.

The settings of UPL-EA are specified as follows: $K = 125$, $\beta=0.5$, $\gamma=1$, $\lambda=10$, and $M=3$. 
The embedding dimension $d$ is set to 300. For BERT pre-trained word embeddings, we use a PCA-based technique~\citep{raunak2019effective} to reduce feature dimension from 768 to 300 with minimal information loss. The number of pseudo-labeling \rv{iterations is set to 9, where each iteration contains 10 training epochs for the EA model}. We implement our model in PyTorch, using the Adam optimizer with a learning rate of 0.001 on DBP15K and DBP-YG-15K, and 0.00025 on SRPRS. The batch size is set to 256. All experiments are run on a computer with an Intel(R) Core(TM) i9-13900KF CPU @ 3.00GHz and an NVIDIA Geforce RTX 4090 (24GB memory) GPU. 


\subsection{Comparison with State-of-the-Art Baselines}
\rv{To thoroughly validate the effectiveness and applicability of UPL-EA, we compare it with a series of state-of-the-art baselines on both cross-lingual and cross-source monolingual datasets.}

\subsubsection{\rv{Results on Cross-Lingual Datasets}}
\rv{On cross-lingual datasets, language differences could impact the difficulty of aligning entities across different linguistic KGs. Following the survey paper by \cite{zhao2020experimental}, we use the term of linguistic barriers to indicate inherent differences in language uses related to syntactic structures, semantic divergences and cultural nuances encoded in languages~\citep{motschenbacher2022linguistic}. For example, DBP15K\textsubscript{ZH\_EN} (Chinese-English) and DBP15K\textsubscript{JA\_EN} (Japanese-English) are considered as distantly-related languages with larger language barriers. 
whereas DBP15K\textsubscript{FR\_EN} (French-English), SRPRS\textsubscript{EN\_FR} (English-French) and SRPRS\textsubscript{EN\_DE} are considered as closely-related languages with smaller language barriers. } 

\begin{table*}[t]
\scriptsize
\centering
\tabcolsep 3.5pt
   \caption{Performance comparison on DBP15K. \rv{The asterisk (*) indicates that semantic meanings of entity names are used to construct entity features. The best and second best results per column are highlighted in \textbf{bold} and \underline{underlined}, respectively. }}
    \label{tab:comparison_dbp15k}
\begin{tabular}{l|c c c|c c c|c c c}
      \toprule 
          \multirow{2}{*}{Models} & \multicolumn{3}{c|}{DBP15K\textsubscript{ZH\_EN}} & \multicolumn{3}{c|}{DBP15K\textsubscript{JA\_EN}} & \multicolumn{3}{c}{DBP15K\textsubscript{FR\_EN}} \\ 
      \cmidrule(l{0em}r{0em}){2-10} 
        &Hit@1 & Hit@10 & MRR & Hit@1  & Hit@10 & MRR & Hit@1 & Hit@10 & MRR \\
      \midrule 
      MtransE & 20.9 & 51.2 & 0.31 & 25.0 & 57.2 & 0.36 & 24.7 & 57.7 & 0.36 \\
      JAPE-Stru& 37.2 & 68.9 & 0.48 & 32.9 & 63.8 & 0.43 & 29.3 & 61.7 & 0.40 \\ 
      GCN-Stru& 39.8 & 72.0 & 0.51 & 40.0 & 72.9 & 0.51 & 38.9 & 74.9 & 0.51\\ 
      JAPE* & 41.4 & 74.1 & 0.53 & 36.5 & 69.5 & 0.48 & 31.8 & 66.8 & 0.44\\
      GCN-Align* & 43.4 & 76.2 & 0.55 & 42.7 & 76.2 & 0.54 & 41.1 & 77.2 & 0.53\\
      HMAN* & 56.1 & 85.9 & 0.67 & 55.7 & 86.0 & 0.67 & 55.0 & 87.6 & 0.66\\      
      RDGCN* & 69.7 & 84.2 & 0.75 & 76.3 & 89.7 & 0.81 & 87.3 & 95.0 & 0.90\\
      HGCN* & 70.8 & 84.0 & 0.76 & 75.8 & 88.9 & 0.81 & 88.8 & 95.9 & 0.91\\
      CEA* & 78.7 & - & - & 86.3 & - & - & 97.2 & - & -\\
      \midrule 
      IPTransE & 33.2 & 64.5 & 0.43 & 29.0 & 59.5 & 0.39 & 24.5 & 56.8 & 0.35\\
      BootEA & 61.4 & 84.1 & 0.69 & 57.3 & 82.9 & 0.66 & 58.5 & 84.5 & 0.68\\
      MRAEA & 75.7 & 93.0 & 0.83 & 75.8 & 93.4 & 0.83 & 78.0 & 94.8 & 0.85\\
      RNM* & 84.0 & 91.9 & 0.87 & 87.2 & 94.4 & 0.90 & 93.8 & 98.1 & 0.95\\
      CPL-OT* & \underline{92.7} & \underline{96.4} & \underline{0.94} & \underline{95.6} & \underline{98.3} & \underline{0.97} & \underline{99.0} & \underline{99.4} & \underline{0.99}\\ 
     \midrule 
      UPL-EA* & \rvtwo{\textbf{94.7}} & \rvtwo{\textbf{97.5}} & \rvtwo{\textbf{0.96}} & \rvtwo{\textbf{97.4}} & \rvtwo{\textbf{98.9}} & \rvtwo{\textbf{0.98}} & \rvtwo{\textbf{99.4}} & \rvtwo{\textbf{99.7}} & \rvtwo{\textbf{1.00}}\\
      \bottomrule 
    \end{tabular}
\end{table*}

\begin{table*}[t]
\scriptsize
\centering
\tabcolsep 10pt
   \caption{Performance comparison on SRPRS. \rv{The asterisk (*) indicates that semantic meanings of entity names are used to construct entity features. The best and second best results per column are highlighted in \textbf{bold} and \underline{underlined}, respectively. }}
    \label{tab:comparison_srprs}
    \begin{tabular}{l|c c c|c c c}
      \toprule 
          \multirow{2}{*}{Models} & \multicolumn{3}{c|}{SRPRS\textsubscript{EN\_FR}} & \multicolumn{3}{c}{SRPRS\textsubscript{EN\_DE}}\\ 
      \cmidrule(l{0em}r{0em}){2-7} 
        &Hit@1 & Hit@10 & MRR & Hit@1  & Hit@10 & MRR \\
      \midrule 
      MtransE & 21.3 & 44.7 & 0.29 & 10.7 & 24.8 & 0.16\\
      JAPE-Stru & 24.1 & 53.3 & 0.34 & 30.2 & 57.8 & 0.40\\ 
      GCN-Stru & 24.3 & 52.2 & 0.34 & 38.5 & 60.0 & 0.46\\ 
      JAPE* & 24.1 & 54.4 & 0.34 & 26.8 & 54.7 & 0.36\\
      GCN-Align* & 29.6 & 59.2 & 0.40 & 42.8 & 66.2 & 0.51\\
      HMAN*& 40.0 & 70.5 & 0.50 & 52.8 & 77.8 & 0.62\\      
      RDGCN* & 67.2 & 76.7 & 0.71 & 77.9 & 88.6 & 0.82\\
      HGCN* & 67.0 & 77.0 & 0.71 & 76.3 & 86.3 & 0.80\\
      CEA* & 96.2 & - & - & 97.1 & - & - \\
      \midrule 
      IPTransE & 12.4 & 30.1 & 0.18 & 13.5 & 31.6 & 0.20\\
      BootEA & 36.5 & 64.9 & 0.46 & 50.3 & 73.2 & 0.58\\
      MRAEA & 46.0 & 76.8 & 0.56 & 59.4 & 81.5 & 0.66\\
      RNM* & 92.5 & 96.2 & 0.94 & 94.4 & 96.7 & 0.95\\
      CPL-OT* & \underline{97.4} & \underline{98.8} & \underline{0.98} & \underline{97.4} & \underline{98.9} & \underline{0.98}\\ 
     \midrule 
     UPL-EA* & \rvtwo{\textbf{98.2}} & \rvtwo{\textbf{99.3}} & \rvtwo{\textbf{0.99}} & \rvtwo{\textbf{98.4}} & \rvtwo{\textbf{99.5}} & \rvtwo{\textbf{0.99}}\\
      \bottomrule 
    \end{tabular}
\end{table*}

\noindent\textbf{Baselines and Metrics.} We compare UPL-EA against 12 state-of-the-art EA models on \rv{cross-lingual datasets: DBP15K and SRPRS, which broadly fall into two categories:}
\begin{itemize}
\item Supervised models, including MTransE~\citep{chen2016multilingual}, JAPE~\citep{sun2017cross}, JAPE in its structure-only variant denoted as JAPE-Stru, GCN-Align~\citep{wang2018cross}, GCN-Align in its structure-only variant denoted as GCN-Stru, RDGCN~\citep{wu2019relation}, HGCN~\citep{wu2019jointly}, HMAN~\citep{yang2019aligning}, and CEA~\citep{zeng2020collective};

\item Pseudo-labeling-based models, including IPTransE~\citep{zhu2017iterative}, BootEA~\citep{sun2018bootstrapping}, MRAEA~\citep{mao2020mraea},  RNM~\citep{zhu2021relation2}, and CPL-OT~\citep{ding2022conflict}. 
\end{itemize}

The results of MRAEA and CPL-OT on both datasets, and RNM on DBP15K are obtained from their original papers. Results of other baselines are obtained from~\citep{zhao2020experimental}. For UPL-EA, we report the average results over five runs. 

\rv{Following the evaluation protocols of mainstream state-of-the-art EA models, we utilize ranking-based metrics: Hit@k ($\text{k}=1,10$) and Mean Reciprocal Rank (MRR), on cross-lingual datasets.
Given a set of test alignments $\mathcal{S}_{\text{test}}$, Hit@k measures the percentage of correctly aligned entity pairs where the true corresponding counterpart of a source entity appears within the top-k positions in the list of candidate counterparts. 
MRR measures the average of the reciprocal ranks for the correctly aligned entities. Higher Hit@k and MRR scores indicate better EA performance.}

\vspace{0.2cm}
\noindent\textbf{The Results.} Table~\ref{tab:comparison_dbp15k} and Table~\ref{tab:comparison_srprs} report performance comparisons on DBP15K and SRPRS, respectively. This set of results is reported with 30\% seed alignments used for training. The asterisk (*) indicates that  semantic meanings of entity names are used to construct entity features. 

Our results show that UPL-EA significantly outperforms most existing EA models on five cross-lingual KG pairs. In particular, on DBP15K\textsubscript{ZH\_EN}, UPL-EA outperforms the second and third performers, CPL-OT and RNM, by 2\% and over 10\%, respectively, in terms of Hit@1. This affirms the necessity of systematically combating confirmation bias for pseudo-labeling-based entity alignment. It is worth noting that disparities in overall performance can be observed among the five cross-lingual KG pairs, where the lowest accuracy is achieved on  DBP15K\textsubscript{ZH\_EN} due to its large linguistic barriers. Nevertheless, for the most challenging EA task on DBP15K\textsubscript{ZH\_EN}, UPL-EA yields strong performance gains over other baselines.

\subsubsection{\rv{Results on Cross-Source Monolingual Datasets}}
\label{section: monolingual result}

\noindent\rv{\textbf{Baselines and Metrics.} 
On monolingual dataset DBP-YG-15K, we compare UPL-EA against five EA models, including BootEA~\citep{sun2018bootstrapping}, RDGCN~\citep{wu2019relation}, BERT-INT~\citep{tang2021bert}, TransEdge~\citep{sun2019transedge}, and PARIS+~\citep{leone2022critical}. 
The results of these baselines are obtained from~\citep{leone2022critical}. BootEA, TransEdge, and UPL-EA use relational triplets only, the same as our setting. PARIS+, RDGCN, and BERT-INT use both relational and attribute triplets. For UPL-EA, we report the average results over five runs.}

\rv{For more comprehensive evaluation, we adopt classification-based metrics suggested by~\cite{leone2022critical} on DBP-YG-15K, which are precision, recall, and $F_1$ score. Given a set of alignments inferred by an EA model $\mathcal{S}_{\text{pred}}$ and a set of test alignments $\mathcal{S}_{\text{test}}$, the three classification-based metrics are calculated as follows:
\begin{equation}
\text{Precision}=\frac{|\mathcal{S}_{\text{pred}}\cap\mathcal{S}_{\text{test}}|}{|\mathcal{S}_{\text{pred}}|}, \text{Recall} = \frac{|\mathcal{S}_{\text{pred}}\cap\mathcal{S}_{\text{test}}|}{|\mathcal{S}_{\text{test}}|}, F_1 = 2 \times\frac{\text{Precision} \times \text{Recall}}{\text{Precision} + \text{Recall}}.\nonumber
\end{equation}}

\vspace{0.2cm}
\noindent\rv{\textbf{The Results.} 
Table~\ref{tab:comparison_dbyg15k} reports performance comparisons on DBP-YG-15K (OpenEA) and DBP-YG-15K (RealEA) with 30\% seed alignments used for training. \rv{The asterisk ``*" indicates that semantic meanings of entity names are used in EA modeling}. }
\rv{Our results show that UPL-EA outperforms all five baselines across two KG pairs of DBP-YG-15K. On the OpenEA KG pair, UPL-EA achieves nearly perfect performance with all three metrics to be approximately 1, outperforming the second best baseline by more than 2\% on $F_1$ score. On the RealEA KG pair of DBP-YG-15K, \rv{even with duplicated entities in each KG~\citep{leone2022critical}}, UPL-EA performs competitively with an over 5\% improvement on $F_1$ score compared to the second best baseline. }

\begin{table*}[h]
\scriptsize
\tabcolsep 9pt
  \begin{center}
   \caption{\rv{Performance comparison on DBP-YG-15K. The asterisk ``*" indicates that semantic meanings of entity names are used in EA modeling. The best and second best results per column are highlighted in \textbf{bold} and \underline{underlined}, respectively. }}
    \label{tab:comparison_dbyg15k}
    \renewcommand{\arraystretch}{1.1}
    {\begin{tabular}{l|c c c|c c c}
    \toprule 
	\multirow{2}{*}{Models} & \multicolumn{3}{c|}{DBP-YG-15K (OpenEA)} & \multicolumn{3}{c}{DBP-YG-15K (RealEA)}\\ 
     \cmidrule(l{0em}r{0em}){2-7} 
       & Precision & Recall & $F_1$-score & Precision & Recall & $F_1$-score \\
      \midrule 
      TransEdge & 0.367 & 0.212 & 0.268 & 0.335 & 0.203 & 0.253\\
      BootEA & 0.926 & 0.675 & 0.781 & 0.459 & 0.313 & 0.372 \\
      RDGCN* & 0.984 & 0.855 & 0.915 & 0.822 & 0.709 & 0.761 \\
      BERT-INT* & 0.875 & 0.969 & 0.920 & 0.817 & 0.827 & 0.822\\
      PARIS+* & \underline{0.998} & \underline{0.961} & \underline{0.979} & \underline{0.906} & \underline{0.931} & \underline{0.918}\\ 
      \midrule
      UPL-EA* & \rvtwo{\textbf{1.000}} & \rvtwo{\textbf{1.000}} & \rvtwo{\textbf{1.000}} & \rvtwo{\textbf{0.976}} & \rvtwo{\textbf{0.964}} & \rvtwo{\textbf{0.970}}\\ 
      \bottomrule 
    \end{tabular}}
  \end{center}
\end{table*}

\rv{Our reported results on both cross-lingual and cross-source monolingual datasets thus far demonstrate the viability of UPL-EA across various datasets. We will then focus on cross-lingual datasets DBP15K and SRPRS for subsequent experiments and analyses, as cross-lingual contexts present complexities like linguistic barriers, which are crucial for assessing the efficacy of our proposed framework.}


\subsection{Ablation Studies and Analyses}
This section presents a series of ablation studies and in-depth analyses to validate the effectiveness of our proposed UPL-EA framework. 

\subsubsection{Effectiveness of Different Components}
To assess the importance of various components of the proposed UPL-EA framework, we first conduct a thorough ablation study on five cross-lingual KG pairs from DBP15K and SRPRS. To provide deeper insights, we undertake ablation studies under two settings: the conventional setting using 30\% seed alignments and the setting with no seed alignments provided. The full UPL-EA model is compared with its ablated variants, with the best performance highlighted by \textbf{bold}. From Table~\ref{tab:ablation_dbp15k} and Table~\ref{tab:ablation_srprs}, we can see that the full UPL-EA model performs the best in all cases. 


\begin{table*}[h!]
\scriptsize
\tabcolsep 1.5pt
  \begin{center}
   \caption{Ablation study on DBP15K}
    \label{tab:ablation_dbp15k}
    \renewcommand{\arraystretch}{1.03}
\begin{tabular}{l|c c c|c c c|c c c}
      \toprule 
      \multirow{2}{*}{Models} & \multicolumn{3}{c|}{DBP15K\textsubscript{ZH\_EN}} & \multicolumn{3}{c|}{DBP15K\textsubscript{JA\_EN}} & \multicolumn{3}{c}{DBP15K\textsubscript{FR\_EN}} \\ 
      \cmidrule(l{0em}r{0em}){2-10} 
        &Hit@1 & Hit@10 & MRR & Hit@1  & Hit@10 & MRR & Hit@1 & Hit@10 & MRR \\
      \midrule 
       \multicolumn{10}{c}{30\% seed alignments}\\
      \midrule 
      Full Model & \rvtwo{\textbf{94.7}} & \rvtwo{\textbf{97.5}} & \rvtwo{\textbf{0.96}} & \rvtwo{\textbf{97.4}} & \rvtwo{\textbf{98.9}} & \rvtwo{\textbf{0.98}} & \rvtwo{\textbf{99.4}} & \rvtwo{\textbf{99.7}} & \rvtwo{\textbf{1.00}}\\ 
      w.o. OT. Pseudo-Labeling  & 80.1 & 90.4 & 0.84 & 87.6 & 94.8 & 0.90 & 94.4 & 97.4 & 0.96\\
      \rvtwo{w.o. Parallel Ensembling} & \rvtwo{84.9} & \rvtwo{90.6} & \rvtwo{0.87} & \rvtwo{91.1} & \rvtwo{95.5} & \rvtwo{0.93} & \rvtwo{98.1} & \rvtwo{99.0} & \rvtwo{0.98}\\
      w.o. OT. \& Ensembling & 74.8 & 87.3 & 0.80 & 83.4 & 93.6 & 0.87 & 93.1 & 97.5 & 0.95\\ 
      \rv{w.o. Dist. Rectification} & \rv{82.6} & \rv{89.6} & \rv{0.85} & \rv{89.8} & \rv{94.5} & \rv{0.91} & \rv{96.8} & \rv{98.0} & \rv{0.97}\\
      \midrule 
      \multicolumn{10}{c}{No seed alignments}\\
      \midrule 
      Full Model & \rvtwo{\textbf{93.0}} & \rvtwo{\textbf{96.2}} & \rvtwo{\textbf{0.94}} & \rvtwo{\textbf{96.0}} & \rvtwo{\textbf{98.3}} & \rvtwo{\textbf{0.97}} & \rvtwo{\textbf{99.2}} & \rvtwo{\textbf{99.5}} & \rvtwo{\textbf{0.99}} \\ 
      w.o. OT. Pseudo-Labeling & 66.9 & 75.5 & 0.70 & 76.9 & 85.2 & 0.80 & 91.9 & 95.2 & 0.93\\
      \rvtwo{w.o. Parallel Ensembling} & \rvtwo{83.0} & \rvtwo{88.7} & \rvtwo{0.85} & \rvtwo{90.1} & \rvtwo{94.6} & \rvtwo{0.92} & \rvtwo{97.8} & \rvtwo{98.8} & \rvtwo{0.98}\\
      w.o. OT. \& Ensembling & 67.1 & 76.8 & 0.71 & 77.1 & 86.2 & 0.81 & 91.4 & 95.9 & 0.93\\ 
      \rv{w.o. Dist. Rectification} & \rv{72.5} & \rv{79.9} & \rv{0.75} & \rv{82.5} & \rv{89.2} & \rv{0.85} & \rv{95.4} & \rv{97.2} & \rv{0.96}\\
      \bottomrule 
    \end{tabular}
  \end{center}
\end{table*}

\begin{table*}[h!]
\scriptsize
  \begin{center}
   \caption{Ablation study on SRPRS}
    \label{tab:ablation_srprs}
        \renewcommand{\arraystretch}{1.1}
\begin{tabular}{l|c c c|c c c}
      \toprule 
      \multirow{2}{*}{Models} & \multicolumn{3}{c|}{SRPRS\textsubscript{EN\_FR}} & \multicolumn{3}{c}{SRPRS\textsubscript{EN\_DE}}\\ 
      \cmidrule(l{0em}r{0em}){2-7} 
        &Hit@1 & Hit@10 & MRR & Hit@1  & Hit@10 & MRR \\
      \midrule 
       \multicolumn{7}{c}{30\% seed alignments}\\
      \midrule 
      Full Model  & \rvtwo{\textbf{98.2}} & \rvtwo{\textbf{99.3}} & \rvtwo{\textbf{0.99}} & \rvtwo{\textbf{98.4}} & \rvtwo{\textbf{99.5}} & \rvtwo{\textbf{0.99}}\\ 
      w.o. OT. Pseudo-Labeling & 93.9 & 96.6 & 0.95 & 94.2 & 97.5 & 0.95\\
      \rvtwo{w.o. Parallel Ensembling} & \rvtwo{94.6} & \rvtwo{97.4} & \rvtwo{0.96} & \rvtwo{94.8} & \rvtwo{98.1} & \rvtwo{0.96}\\
      w.o. OT. \& Ensembling & 92.7 & 96.5 & 0.94 & 93.6 & 97.4 & 0.95\\ 
      \rv{w.o. Dist. Rectification} & \rv{95.1} & \rv{96.7} & \rv{0.96} & \rv{97.0} & \rv{98.2} & \rv{0.97}\\
      \midrule 
      \multicolumn{7}{c}{No seed alignments}\\
      \midrule 
      Full Model & \rvtwo{\textbf{97.9}} & \rvtwo{\textbf{99.2}} & \rvtwo{\textbf{0.98}} & \rvtwo{\textbf{97.4}} & \rvtwo{\textbf{99.2}} & \rvtwo{\textbf{0.98}}\\ 
      w.o. OT. Pseudo-Labeling & 89.8 & 93.0 & 0.91 & 91.4 & 95.2 & 0.93 \\
      \rvtwo{w.o. Parallel Ensembling} & \rvtwo{94.2} & \rvtwo{97.0} & \rvtwo{0.95} & \rvtwo{94.8} & \rvtwo{97.7} & \rvtwo{0.96}\\
      w.o. OT. \& Ensembling & 89.7 & 93.1 & 0.91 & 91.1 & 94.9 & 0.93\\
      \rv{w.o. Dist. Rectification} & \rv{93.0} & \rv{94.7} & \rv{0.94} & \rv{94.9} & \rv{97.0} & \rv{0.96}\\
      \bottomrule 
    \end{tabular}
  \end{center}
\end{table*}

\begin{itemize} 

\item\textbf{w.o. OT. Pseudo-Labeling:} To study the efficacy of OT-based pseudo-labeling, we ablate it from the full UPL-EA model. As OT modeling can effectively eliminate a considerable number of \rvtwo{conflicted misalignments} to ensure one-to-one correspondences in pseudo-labeled alignments, this ablation results in a profound performance drop across all datasets on both settings.

\item\textbf{\rvtwo{w.o. Parallel Ensembling:}} The ablation of parallel pseudo-label ensembling from UPL-EA also significantly degrades alignment performance, with substantial performance declines observed in both settings, especially on DBP15K\textsubscript{ZH\_EN} and DBP15K\textsubscript{JA\_EN} with large linguistic barriers. \rv{In contrast, performance drops are less pronounced on DBP15K\textsubscript{FR\_EN}, SRPRS\textsubscript{EN\_FR}, and SRPRS\textsubscript{EN\_DE} with relatively small linguistic barriers. This is attributed to the fact that larger linguistic barriers tend to incur more one-to-one misalignments. Our findings confirm that parallel pseudo-label ensembling is crucial for UPL-EA to achieve its full potential, especially when model predictions are less accurate during early training stages.}

\item\textbf{w.o. OT. \& Ensembling:} We also analyze the overall effect of ablating both OT modeling and parallel pseudo-label ensembling from the full model. This ablation, \rv{conceptually identical to the naive pseudo-labeling strategy~\citep{sun2019transedge}}, has a substantial adverse impact, leading to a dramatic performance drop in all cases. Our results highlight the effectiveness of our proposed UPL-EA framework in systematically combating confirmation bias for pseudo-labeling-based entity alignment. 

\item\rv{\textbf{w.o. Dist. Rectification:} The effectiveness of embedding distance rectification is examined by using the original embedding distance defined in Eq.~(\ref{eq:dist}) as the transport cost used for OT modeling. The ablation of distance rectification leads to a significant performance drop at both settings. This highlights the complementary role of distance rectification in the training of the EA model, particularly during the early stages, for learning more informative entity embeddings and providing a reliable cost measure for OT modeling. }
\end{itemize}

Note that under the setting with no seed alignments, the variant without OT-based pseudo-labeling (w.o. OT. Pseudo-Labeling) has similar performance as compared to the variant completely ignoring confirmation bias (w.o. OT. \& Ensemb.). In particular, on DBP15K\textsubscript{ZH\_EN} and DBP15K\textsubscript{JA\_EN}, the former variant even performs slightly worse. This is because under the challenging case where there are no seed alignments, ablating OT-based pseudo-labeling might incur considerably more conflicted misalignments.
As a result, it becomes ineffective to filter out erroneous pseudo-labeled alignments via ensembling.

\subsubsection{\rv{Comparisons with Other Ensembling Methods}}
\label{subsection:ensembling}


To investigate the effectiveness of UPL-EA's parallel pseudo-label ensembling, we carry out a case study on DBP15K using 30\% seed alignments. Specifically, we compare the performance of UPL-EA with three ensembling methods: \rvtwo{(1) Parallel pseudo-label ensembling, (2) Parallel pseudo-label ensembling with majority vote, and (3) Temporal ensembling~\citep{laine2017temporal}. The entity alignment performance of UPL-EA using the three ensembling methods is reported in Table~\ref{tab:different_ensembling}. Our results show that UPL-EA using our proposed parallel ensembling (UPL-EA$_{\text{P.E.}}$) consistently outperforms the variant using majority vote (UPL-EA$_{\text{M.V.}}$) and the variant using temporal ensembling (UPL-EA$_{\text{T.E.}}$). The performance advantage is particularly significant on DBP15K\textsubscript{ZH\_EN} and DBP15K\textsubscript{JA\_EN}, where larger linguistic barriers exist. Specifically, UPL-EA$_{\text{T.E.}}$ performs the worst across all datasets, as self-ensembling approaches impose cross-iteration dependencies, exacerbating error propagation in the context of pseudo-labeling. Our findings suggest that UPL-EA's parallel pseudo-label ensembling provides a simple but effective way to improve the quality of pseudo-labeled alignments, achieving competitive performance compared to other ensembling methods. } 

\begin{table*}[t]
\scriptsize
\tabcolsep 4pt
  \begin{center}
   \caption{\rv{Performance of UPL-EA using different ensembling methods.}}
    \label{tab:different_ensembling}
    \renewcommand{\arraystretch}{1.03}
    {\begin{tabular}{l|c c c|c c c|c c c}
    \toprule 
	\multirow{2}{*}{} & \multicolumn{3}{c|}{DBP15K\textsubscript{ZH\_EN}} & \multicolumn{3}{c|}{DBP15K\textsubscript{JA\_EN}} & \multicolumn{3}{c}{DBP15K\textsubscript{FR\_EN}}\\ 
     \cmidrule(l{0em}r{0em}){2-10} 
       &Hit@1 & Hit@10 & MRR & Hit@1  & Hit@10 & MRR & Hit@1 & Hit@10 & MRR\\
      \midrule 
      UPL-EA$_{\text{P.E.}}$ & \rvtwo{\textbf{94.7}} & \rvtwo{\textbf{97.5}} & \rvtwo{\textbf{0.96}} & \rvtwo{\textbf{97.4}} & \rvtwo{\textbf{98.9}} & \rvtwo{\textbf{0.98}} & \rvtwo{\textbf{99.4}} & \rvtwo{\textbf{99.7}} & \rvtwo{\textbf{1.00}}\\
      UPL-EA$_{\text{M.V.}}$ & \rvtwo{93.3} & \rvtwo{96.6} & \rvtwo{0.95} & \rvtwo{96.0} & \rvtwo{98.2} & \rvtwo{0.97} & \rvtwo{99.0} & \rvtwo{99.4} & \rvtwo{0.99}\\
      UPL-EA$_{\text{T.E.}}$ & 92.8 & 95.3 & 0.94 & 96.0 & 97.6 & 0.97 & 98.7 & 99.0 & 0.99 \\
      \bottomrule 
    \end{tabular} }
  \end{center}
\end{table*}

\subsubsection{\rv{Effectiveness as a General Pseudo-Labeling Framework}}
\label{subsection:framework}


\rv{To further demonstrate UPL-EA's viability as a general pseudo-labeling framework for entity alignment, we substitute the EA model in UPL-EA (described in Section~\ref{section:EAmodel}) with alternative EA models, and examine if applying our UPL strategy could bring any performance improvements. We consider two alternative EA models: (1) GCN-Align~\citep{wang2018cross}, which adopts a two-layer GCN as an encoder to learn entity embeddings, and (2) GAT-Align, where the GCN encoder in GCN-Align is replaced with a two-layer GAT for embedding learning. Both EA models use the same loss function provided in Eq.~(\ref{eq:hard_loss}).  
This analysis is conducted on DBP15K with 30\% seed alignments as a case study. The entity alignment performance using the two baselines and their UPL-EA augmented counterparts is reported in Table~\ref{tab:improve gnn}.}

\begin{table*}[h!]
\scriptsize
\tabcolsep 4pt
  \begin{center}
   \caption{\rv{Performance of UPL-EA instantiated with other EA models}}
    \label{tab:improve gnn}
    \renewcommand{\arraystretch}{1.03}
    {\begin{tabular}{l|c c c|c c c|c c c}
    \toprule 
	\multirow{2}{*}{} & \multicolumn{3}{c|}{DBP15K\textsubscript{ZH\_EN}} & \multicolumn{3}{c|}{DBP15K\textsubscript{JA\_EN}} & \multicolumn{3}{c}{DBP15K\textsubscript{FR\_EN}}\\ 
     \cmidrule(l{0em}r{0em}){2-10} 
       &Hit@1 & Hit@10 & MRR & Hit@1  & Hit@10 & MRR & Hit@1 & Hit@10 & MRR\\
      \midrule 
      GCN-Align & 43.4 & 76.2 & 0.55 & 42.7 & 76.2 & 0.54 & 41.1 & 77.2 & 0.53\\
      GCN$_{\text{UPL-EA}}$ & \rvtwo{79.6} & \rvtwo{91.5} & \rvtwo{0.84} & \rvtwo{82.9} & \rvtwo{94.2} & \rvtwo{0.87} & \rvtwo{87.0} & \rvtwo{96.8} & \rvtwo{0.91}\\
      \midrule 
      GAT-Align & 71.3 & 84.3 & 0.76 & 81.2 & 91.9 & 0.85 & 92.9 & 97.9 & 0.95\\
      GAT$_{\text{UPL-EA}}$ & \rvtwo{92.0} & \rvtwo{96.8} & \rvtwo{0.94} & \rvtwo{93.8} & \rvtwo{98.1} & \rvtwo{0.95} & \rvtwo{98.3} & \rvtwo{99.6} & \rvtwo{0.99}\\
      \midrule 
      UPL-EA & \rvtwo{\textbf{94.7}} & \rvtwo{\textbf{97.5}} & \rvtwo{\textbf{0.96}} & \rvtwo{\textbf{97.4}} & \rvtwo{\textbf{98.9}} & \rvtwo{\textbf{0.98}} & \rvtwo{\textbf{99.4}} & \rvtwo{\textbf{99.7}} & \rvtwo{\textbf{1.00}}\\
      \bottomrule 
    \end{tabular} }
  \end{center}
\end{table*}

\rv{Our results in Table~\ref{tab:improve gnn} indicate that applying UPL-EA to both GCN-Align and GAT-Align improves entity alignment performance by a considerable margin, with an average 20\% improvement in Hit@1. Our results affirm the strong modular utility of UPL-EA as a general pseudo-labeling framework in boosting various EA models to achieve better alignment performance. }

\subsection{Comparison w.r.t. Different Rates of Seed Alignments}

Next, we further \rv{examine how the performance of UPL-EA changes with respect to different rates of seed alignments, decreasing from 40\% to 10\%.} We compare UPL-EA with four representative state-of-the-art baselines (BootEA, RDGCN, RNM, and CPL-OT), and report the results on DBP15K and \rv{SRPRS in Table~\ref{tab:labelrate_dbp15k} and Table~\ref{tab:labelrate_srprs}. The last column ``$\Delta\downarrow$" in each table indicates the average performance loss when decreasing the rate from 40\% to 10\% for each model.} 

\begin{table*}[h!]
\scriptsize
\centering
\tabcolsep 2.8pt
   \caption{Performance comparison (Hit@1) on DBP15K with respect to different rates of seed alignments. \rv{``$\Delta\downarrow$" in the last column indicates the average performance loss when decreasing the rate from 40\% to 10\% on three datasets. }}
    \label{tab:labelrate_dbp15k}
\begin{tabular}{l|c c c c|c c c c|c c c c |c}
      \toprule 
          \multirow{3}{*}{Models} & \multicolumn{4}{c|}{DBP15K\textsubscript{ZH\_EN}} & \multicolumn{4}{c|}{DBP15K\textsubscript{JA\_EN}} & \multicolumn{4}{c|}{DBP15K\textsubscript{FR\_EN}} & \multirow{3}{*}{\rv{$\Delta\downarrow$ }}\\ 
      \cmidrule(l{0em}r{0em}){2-13} 
      \cmidrule(l{0em}r{0em}){2-13} 
        &40\% & 30\% & 20\% & 10\% &40\% & 30\% & 20\% & 10\% & 40\% & 30\% & 20\% & 10\% \\
      \midrule 
      BootEA & 67.9 & 62.9 & 57.3 & 45.7 & 66.0 & 62.2 & 53.5 & 42.9 & 68.6 & 65.3 & 59.8 & 47.3 & \rv{-22.2}\\
      RDGCN & 72.6 & 70.8 & 68.9 & 66.6 & 79.0 & 76.7 & 74.5 & 72.4 & 89.7 & 88.6 & 87.6 & 86.3 & \rv{-5.3}\\
      RNM & 85.4 & 84.0 & 81.7 & 79.3 & 88.8 & 87.2 & 85.9 & 83.4 & 94.5 & 93.8 & 93.0 & 92.3 & \rv{-4.6}\\
      CPL-OT & 93.0 & 92.7 & 92.2 & 91.8 & 96.1 & 95.6 & 95.1 & 94.7 & 99.2 & 99.1 & 98.9 & 98.7 & \rv{-1.0}\\
     \midrule 
     UPL-EA & \rvtwo{\textbf{95.0}} & \rvtwo{\textbf{94.7}} & \rvtwo{\textbf{94.2}} & \rvtwo{\textbf{93.6}} & \rvtwo{\textbf{97.6}} & \rvtwo{\textbf{97.4}} & \rvtwo{\textbf{97.0}} & \rvtwo{\textbf{96.6}} & \rvtwo{\textbf{99.5}} & \rvtwo{\textbf{99.4}} & \rvtwo{\textbf{99.4}} & \rvtwo{\textbf{99.2}} & \rvtwo{-1.0}\\
      \bottomrule 
    \end{tabular}
\end{table*}

\begin{table*}[h!]
\scriptsize
\centering
\tabcolsep 7pt
   \caption{\rv{Performance comparison (Hit@1) on SRPRS with respect to different rates of seed alignments. \rv{``$\Delta\downarrow$" in the last column indicates the average performance loss when decreasing the rate from 40\% to 10\% on two datasets. }}}
    \label{tab:labelrate_srprs}

{\begin{tabular}{l|c c c c|c c c c | c}
      \toprule 
          \multirow{3}{*}{Models} & \multicolumn{4}{c|}{SRPRS\textsubscript{EN\_FR}} & \multicolumn{4}{c|}{SRPRS\textsubscript{EN\_DE}} & \multirow{3}{*}{$\Delta\downarrow$}\\ 
      \cmidrule(l{0em}r{0em}){2-9} 
      \cmidrule(l{0em}r{0em}){2-9} 
        &40\% & 30\% & 20\% & 10\%& 40\% & 30\% & 20\% & 10\% \\
      \midrule 
      BootEA & 39.9 & 36.5 & 31.1 & 18.3 & 53.6 & 50.3 & 43.3 & 32.8 & -20.7\\
      RDGCN & 68.7 & 67.2 & 65.8 & 64.0 & 79.0 & 77.9 & 76.8 & 75.7 & -3.2\\
      RNM & 93.6 & 92.5 & 90.4 & 89.3 & 95.0 & 94.4 & 93.8 & 92.9 & -2.1\\
      CPL-OT & 97.6 & 97.4 & 97.3 & 97.1 & 97.6 & 97.4 & 97.2 & 97.0 & -0.5\\ 
     \midrule 
     UPL-EA & \rvtwo{\textbf{98.2}} & \rvtwo{\textbf{98.2}} & \rvtwo{\textbf{98.0}} & \rvtwo{\textbf{97.9}} & \rvtwo{\textbf{98.4}} & \rvtwo{\textbf{98.4}} & \rvtwo{\textbf{97.7}} & \rvtwo{\textbf{97.4}} & \rvtwo{-0.7}\\
      \bottomrule 
    \end{tabular}%
    }
\end{table*}

As expected, UPL-EA consistently outperforms four competitors on all cross-lingual KG pairs at all seed alignment rates. This is due to UPL-EA's ability to augment the training set with reliable pseudo-labeled alignments by effectively alleviating confirmation bias. As the rate of seed alignments decreases from 40\% to 10\%, the performance of BootEA significantly degrades by over 20\% on average due to its limited ability to prevent the accumulation of pseudo-labeling errors. RNM outperforms RDGCN owing to its posterior embedding distance editing during pseudo-labeling; however, its lack of iterative model re-training hinders its overall performance. CPL-OT demonstrates more stable performance with varying rates of seed alignments because it selects pseudo-labeled alignments via the conflict-aware OT modeling and then uses them to train the EA model in turn; nevertheless, its neglect of one-to-one misalignments limits the potential of CPL-OT. \rv{UPL-EA remains consistently competitive and stable across all datasets, with an average performance loss of 1\% at most when the rate of seed alignments decreases from 40\% to 10\%, even on the most challenging DBP15K\textsubscript{ZH\_EN} dataset.}




\subsection{Impact of Pre-Trained Word Embeddings}
To analyze the impact of using different pre-trained word embeddings, we report the results of UPL-EA that form entity features with Glove embedding~\citep{pennington2014glove}, which is widely used in the existing EA models. 
We conduct this analysis on the setting with 30\% seed alignments. The results on DBP15K are reported in Table~\ref{tab:different word embs} as a case study. We can observe that UPL-EA with Glove embedding still achieves competitive results, significantly outperforming all other baselines. This confirms that the efficacy of UPL-EA is not highly dependent on embedding initialization methods used. When switching from Glove embedding to BERT pre-trained embeddings, performance gains can be observed, especially on DBP15K\textsubscript{JA\_EN}. This indicates the usefulness of pre-trained word embeddings of high quality for entity alignment.

\begin{table*}[h!]
\scriptsize
\tabcolsep 5pt
  \begin{center}
   \caption{Impact of pre-trained word embeddings}
    \label{tab:different word embs}
    \renewcommand{\arraystretch}{1.03}
\begin{tabular}{l|c c c|c c c|c c c}
    \toprule 
	\multirow{2}{*}{Models} & \multicolumn{3}{c|}{DBP15K\textsubscript{ZH\_EN}} & \multicolumn{3}{c|}{DBP15K\textsubscript{JA\_EN}} & \multicolumn{3}{c}{DBP15K\textsubscript{FR\_EN}}\\ 
     \cmidrule(l{0em}r{0em}){2-10} 
       &Hit@1 & Hit@10 & MRR & Hit@1  & Hit@10 & MRR & Hit@1 & Hit@10 & MRR\\
      \midrule 
      Glove & \rvtwo{93.9} & \rvtwo{97.5} & \rvtwo{0.95} & \rvtwo{96.2} & \rvtwo{98.9} & \rvtwo{0.97} & \rvtwo{99.0} & \rvtwo{99.7} & \rvtwo{0.99}\\
      BERT & \rvtwo{94.7} & \rvtwo{97.5} & \rvtwo{0.96} & \rvtwo{97.4} & \rvtwo{98.9} & \rvtwo{0.98} & \rvtwo{99.4} & \rvtwo{99.7} & \rvtwo{1.00}\\
      \bottomrule 
    \end{tabular}
  \end{center}
\end{table*}

\subsection{Hyper-Parameter Sensitivity Analysis}

We further study the sensitivity of UPL-EA with regards to four hyper-parameters: embedding dimension $d$, number of \rvtwo{OT-based models $M$ for pseudo-label ensembling}, regularization hyper-parameter $\beta$ in Eq.~(\ref{eq:sinkhorn_ot_problem}), and margin hyper-parameter $\gamma$ in the alignment loss function Eq.~(\ref{eq:hard_loss}). This set of sensitivity analysis is conducted on DBP15K\textsubscript{ZH\_EN} with 30\% seed alignments as a case study. The respective results in terms of Hit@1 and Hit@10 are reported in Fig.~\ref{fig:sensitivity}.

\begin{figure*}[h!]
     \centering
     \begin{subfigure}[b]{0.45\textwidth}
         \centering
         \includegraphics[width=\textwidth]{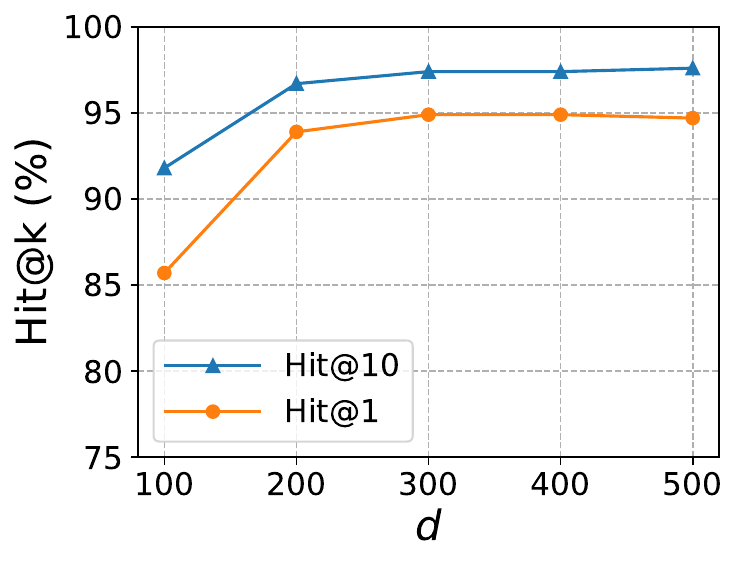}
         \caption{Embedding dimension $d$}
         \label{fig:dimension}
     \end{subfigure}
     \hspace{0.8cm}
         \begin{subfigure}[b]{0.45\textwidth}
         \centering
         \includegraphics[width=\textwidth]{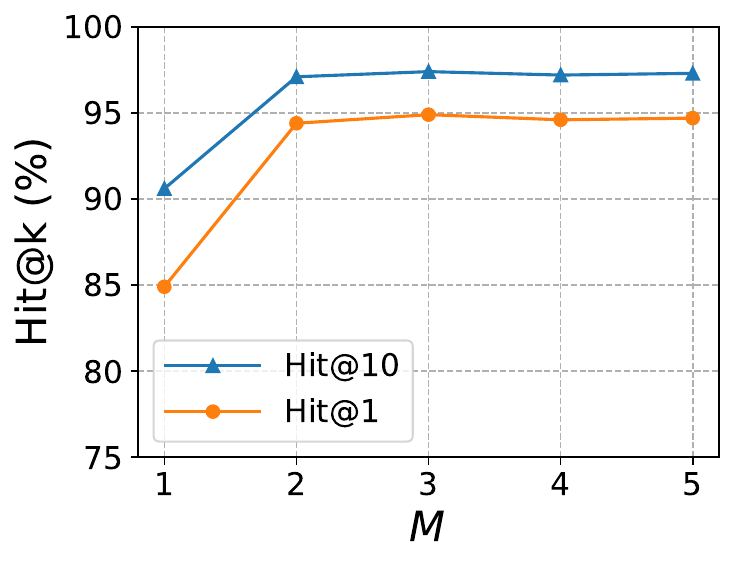}
         \caption{Number of \rvtwo{models} for \rv{ensembling} $M$}
         \label{fig:m}
     \end{subfigure}\\
      \begin{subfigure}[b]{0.45\textwidth}
         \centering
         \includegraphics[width=\textwidth]{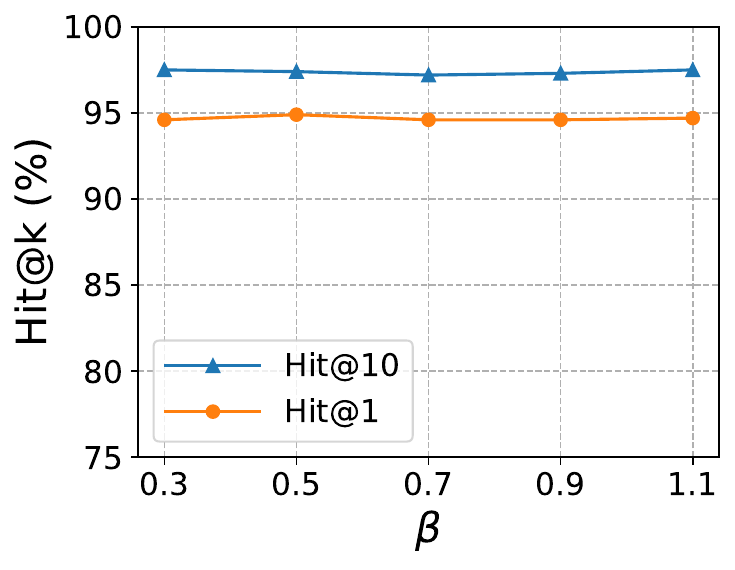}
         \caption{Regularization hyper-parameter $\beta$}
         \label{fig:beta}
     \end{subfigure}
     \hspace{0.8cm}
     \begin{subfigure}[b]{0.45\textwidth}
         \centering
         \includegraphics[width=\textwidth]{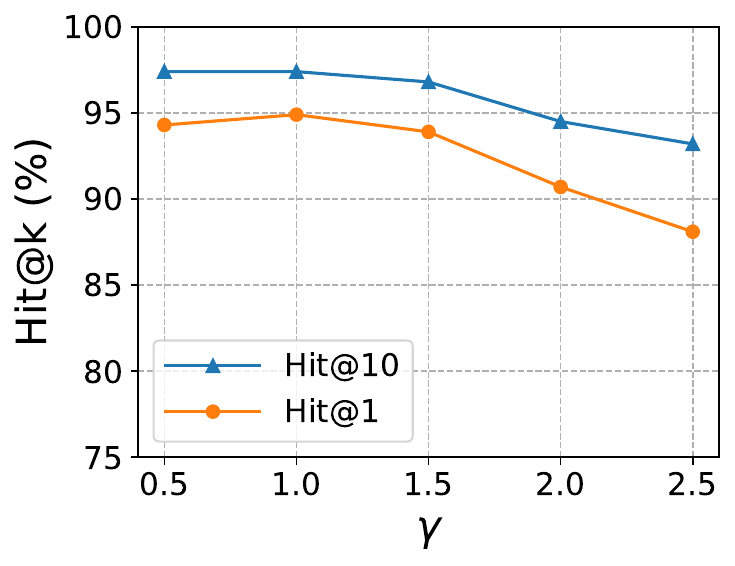}
         \caption{Margin hyper-parameter $\gamma$}
         \label{fig:gamma}
     \end{subfigure}
        \caption{Hyper-parameter sensitivity analysis on DBP15K\textsubscript{ZH\_EN}}
        \label{fig:sensitivity}
\end{figure*}

As shown in Fig.~\ref{fig:dimension}, the performance of UPL-EA improves considerably as the embedding dimension $d$ increases from 100 to 300 and then retains a relatively stable level. 
\rv{For the number of \rvtwo{OT-based models} $M$ used for parallel pseudo-label ensembling, the use of ensembling over multiple \rvtwo{OT-based models} ($M>1$) significantly improves the alignment performance over a single one ($M=1$). This demonstrates the effectiveness of our parallel ensembling mechanism, which requires only a few \rvtwo{OT-based models} (e.g., $M=3$)} to achieve competitive performance (see Fig.~\ref{fig:m}). In addition, Fig.~\ref{fig:beta} shows that the performance of UPL-EA is insensitive to different values of $\beta$ used in OT-based pseudo-labeling. As for the margin parameter $\gamma$, the performance of UPL-EA begins to drop gradually when $\gamma$ exceeds 1, as shown in Fig.~\ref{fig:gamma}. This is reasonable, as a larger margin would allow more tolerance for alignment errors, thereby degrading model performance.


\subsection{\rv{Runtime Comparison}}

Lastly, we compare the overall training time of UPL-EA with \rvtwo{three} embedding-based EA models, including CPL-OT, RDGCN, and RNM, and \rvtwo{one conventional EA model, PARIS+}, across five cross-lingual datasets with 30\% seed alignments. For a fair comparison, we use the same parameters reported in the original papers of the \rvtwo{four} baselines. Fig.~\ref{fig:Time efficiency comparison} reports the overall training time of UPL-EA and the other baselines. Our results show that UPL-EA is considerably more efficient than the three embedding-based baselines, achieving a speedup of at least 50\% across all five datasets. \rvtwo{Although PARIS+ exhibits the shortest runtime due to its rule-based nature and lack of gradient-based optimization, UPL-EA remains highly efficient while maintaining strong EA performance.}

\rv{Notably, CPL-OT and RNM take more than twice as long as UPL-EA, and three times as long on larger datasets such as DBP15K\textsubscript{FR\_EN}. Additionally, the supervised model RDGCN requires over 60-minute training time on DBP15K and almost 30 minutes on SRPRS, indicating its poor runtime efficiency. Overall, our findings suggest that UPL-EA exhibits superior runtime efficiency compared to strong \rvtwo{embedding-based} baselines\rvtwo{, while maintaining a well-balanced trade-off between EA performance and time efficiency compared to conventional EA methods.}}

\begin{figure*}[h!]
\centering
    \includegraphics[width=0.95\textwidth]{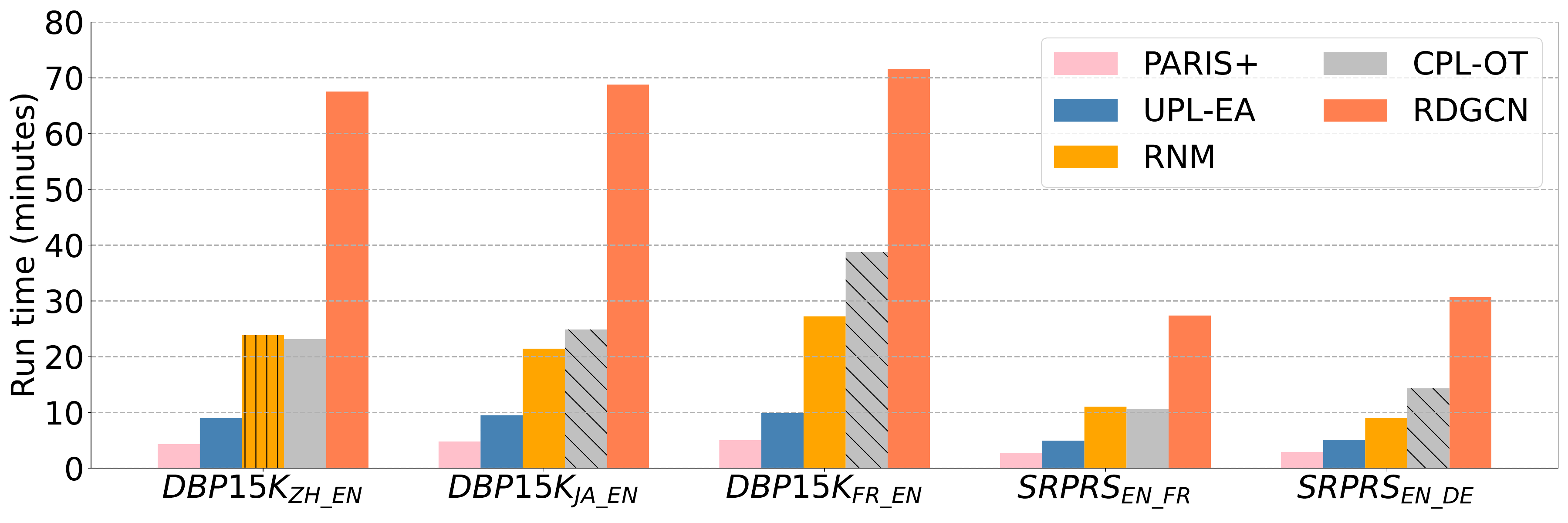}
    \caption{\rvtwo{Runtime comparison}}
    \label{fig:Time efficiency comparison}
\end{figure*}

\section{Related Works}
\label{section:related work}

In this section, we review three streams of related literature, including entity alignment in knowledge graphs, pseudo-labeling in semi-supervised learning, and optimal transport on graphs.

\subsection{Entity Alignment in Knowledge Graphs}

\rv{The entity alignment (EA) task aims to discover one-to-one equivalent entity pairs across two KGs that refer to the same real-world identity. Early EA models are probabilistic methods that compute similarities and perform equivalence reasoning in the input space. 
\rvtwo{For example, PARIS~\citep{suchanek2011paris} is an unsupervised ontology alignment model that jointly aligns entities, relations, and classes across KGs. It converts each relation into a logic-rule based function, and iteratively refines entity alignment probabilities using the functional and inverse functional properties of relations. 
}
PARIS is later extended to a semi-supervised EA model, PARIS+~\citep{leone2022critical}, allowing the incorporation of seed alignments.}

\rv{Since 2017, most EA models are embedding-based}, using distances between entity embeddings in latent spaces to measure the semantic correspondences between entities. Inspired by TransE~\citep{bordes2013translating}, MTransE~\citep{chen2016multilingual} embeds two KGs into two respective embedding spaces, where a transformation matrix is learned using seed alignments. \rv{To obtain better KG embeddings, TransEdge~\citep{sun2019transedge} enhances the translational scoring function by replacing the relation embedding with an edge embedding that incorporates information from both the head and tail entities, in addition to the relation information.} To reduce the number of parameters involved, most subsequent models~\citep{sun2017cross,sun2018bootstrapping,zhu2017iterative} embed KGs into a common latent space by imposing the embeddings of pre-aligned entities to be as close as possible. This ensures that alignment similarities between entities can be directly measured via their embeddings. 

\rv{More recent EA models leverage GNNs to incorporate KG structural information for entity alignment.} For example, GCN-Align~\citep{wang2018cross} adopts GCNs to learn better entity embeddings for alignment inference. However, GCNs and their variants are inclined to result in alignment conflicts, \rv{as their feature aggregation scheme incurs an over-smoothing issue~\citep{min2020scattering, jiang2022sparse}: The embeddings of entities among local neighborhood become indistinguishable similar as the number of GCN layers increases.}
To mitigate the over-smoothing effect, more recent works~\citep{wu2019jointly,wu2019relation,zhu2021relation2} adopt a highway strategy~\citep{srivastava2015highway} on GCN layers, which ``mixes" the learned entity embeddings with the original features. Another line of research efforts is devoted to improving GCN-based approaches through considering heterogeneous relations in KGs.
HGCN~\citep{wu2019jointly} jointly learns the embeddings of entities and relations, without considering the directions of relations. RDGCN~\citep{wu2019relation} performs embedding learning on a dual relation graph, but fails to incorporate statistical information of neighboring relations of an entity. RNM~\citep{zhu2021relation2} uses iterative relational neighborhood matching to refine finalized entity embedding distances. This matching mechanism proves to be empirically effective, but it is used only after the completion of model training and fails to reinforce embedding learning in turn. \rv{BERT-INT~\citep{tang2021bert} leverages the BERT (Bidirectional Encoder Representations from Transformers) model to capture both entity and contextual information from relational paths between entities, thereby enhancing entity alignment performance by incorporating rich semantics such as entity descriptions. Yet, obtaining powerful description information in practice can be challenging in many real-world scenarios.} All the aforementioned models, however, require an abundance of seed alignments provided for training purposes, which are labor-intensive and costly to acquire in real-world KGs. 

To tackle the shortage of seed alignments, semi-supervised EA models have been proposed in recent years. As a prominent learning paradigm among such, pseudo-labeling-based methods, e.g., BootEA~\citep{sun2018bootstrapping}, IPTransE~\citep{zhu2017iterative}, \rv{TransEdge~\citep{sun2019transedge}}, RNM~\citep{zhu2021relation2}, MRAEA~\citep{mao2020mraea}, and CPL-OT~\citep{ding2022conflict}, propose to iteratively pseudo-label unaligned entity pairs and add them to seed alignments for subsequent model training. For RNM, there is a slight difference that it augments seed alignments to rectify embedding distance after the completion of model training. Although these methods have achieved promising performance gains, the confirmation bias associated with iterative pseudo-labeling has been largely under-explored. 
Recent methods like RNM~\citep{zhu2021relation2} and MRAEA~\citep{mao2020mraea} use simple heuristics to preserve only the most convincing alignment pairs, for example, those with the smallest distance, at the presence of conflicts. BootEA~\citep{sun2018bootstrapping} and CPL-OT~\citep{ding2022conflict}, on the other hand, model the inference of pseudo-labeled alignments as an assignment problem, where the most likely aligned pairs are selected at each pseudo-labeling iteration. Unlike BootEA that selects a small set of pseudo-labeled alignments using a pre-specified threshold, CPL-OT imposes a full match between two unaligned entity sets to maximize the number of pseudo-labeled alignments at each iteration. Both methods impose constraints to enforce hard alignments \rv{to alleviate alignment conflicts, which may potentially increase one-to-one misalignments.} 

\rv{This work is thus proposed to explicitly address confirmation bias in pseudo-labeling-based entity alignment. We analytically identify two types of pseudo-labeling errors that lead to confirmation bias and propose a new UPL-EA framework to alleviate these errors. Different from our previous work CPL-OT~\citep{ding2022conflict}, UPL-EA introduces a discrete OT formulation aimed at addressing \rvtwo{conflicted misalignments}. This formulation allows for a more accurate, probabilistic alignment configuration optimized efficiently using the Sinkhorn algorithm. Unlike CPL-OT, which relies on a pre-specified threshold, the threshold for selecting pseudo-labeled alignments in UPL-EA is mathematically derived and proven empirically effective, facilitating its applicability across various datasets. In addition, a parallel ensembling approach is further proposed to refine pseudo-labeled alignments by combining predictions over multiple \rvtwo{OT-based models trained in parallel}, thus mitigating \rvtwo{one-to-one misalignments}.}

\subsection{Pseudo-Labeling}
Pseudo-labeling has emerged as an effective semi-supervised approach in addressing the challenge of label scarcity. It refers to a self-training paradigm where the model is iteratively bootstrapped with additional labeled data based on its own predictions. The pseudo-labels generated from model predictions can be defined as hard (one-hot distribution) or soft (continuous distribution) labels~\citep{lee2013pseudo,shi2018transductive,arazo2020pseudo}. More specifically, pseudo-labeling strategies are designed to select high-confidence unlabeled data by either directly taking the model's predictions, or sharpening the predicted probability distribution. It is closely related to entropy regularization~\citep{sajjadi2016mutual}, where the model's predictions are encouraged to have low entropy (i.e., high-confidence) on unlabeled data. The selected pseudo-labels are then used to augment the training set and to fine-tune the model initially trained on the given labels. This training regime is also extended to an explicit teacher-student configuration~\citep{pham2021meta}, where a teacher network generates pseudo-labels from unlabeled data, which are used to train a student network.

Despite its promising results, pseudo-labeling is inevitably susceptible to erroneous pseudo-labels, thus suffering from confirmation bias~\citep{arazo2020pseudo,rizve2021in}, where the prediction errors would accumulate and degrade model performance. The confirmation bias has been recently studied in the field of computer vision. In works like~\citep{arazo2020pseudo,rizve2021in}, confirmation bias is considered as a problem of poor network calibration, where the network is overfitted towards erroneous pseudo-labels. To alleviate confirmation bias, pseudo-labeling approaches have adopted strategies such as mixup augmentation~\citep{arazo2020pseudo} and uncertainty weighting~\citep{rizve2021in}. Subsequent works like~\citep{cascante2021curriculum,zhang2021flexmatch} address confirmation bias by applying curriculum learning principles, where the decision threshold is adaptively adjusted during the training process and model parameters are re-initialized after each iteration. 

Recently, pseudo-labeling has also been studied on graphs for the task of semi-supervised node classification~\citep{li2018deeper,Sun2020MultiStageSL,Li2023informative}. \citet{li2018deeper} propose a self-trained GCN that enlarges the training set by assigning a pseudo-label to high-confidence unlabeled nodes, and then re-trains the model using both genuine labels and pseudo-labels. The pseudo-labels are generated via a random walk model in a co-training manner. \citet{Sun2020MultiStageSL} show that a shallow GCN is ineffective in propagating label information under few-label settings, and employ a multi-stage self-training approach that relies on a deep clustering model to assign pseudo-labels. \citet{Li2023informative} propose to incorporate the node informativeness scores for the selection of pseudo-labels and adopt distinct loss functions for genuine labels and pseudo-labels during model training. Despite these research efforts, the problem of confirmation bias remains under-explored in graph domains. This work systematically analyzes the cause of confirmation bias and proposes a principled approach to conquer confirmation bias for pseudo-labeling-based entity alignment across KGs.


\subsection{Optimal Transport on Graphs}

Optimal Transport (OT) is the general problem of finding an optimal plan to move one distribution of mass to another with the minimal cost~\citep{villani2009optimal}. As an effective metric to define the distance between probability distributions, OT has been applied in computer vision and natural language processing over a range of tasks including machine translation, text summarization, and image captioning~\citep{torres2021survey,chen2020graph}. In recent years, OT has also been studied on graphs to match graphs with similar structures or align nodes/entities across graphs. For graph partitioning and matching, the transport on the edges across graphs is used to define the Gromov-Wasserstein (GW) discrepancy~\citep{titouan2019optimal} that measures how edges in a graph compare to those in another graph~\citep{xu2019gromov,petric2019got,Xu2019scalable}. 
For entity alignment across graphs, \citet{pei2019improving} incorporate an OT objective into the overall loss to enhance the learning of entity embeddings. \citet{tang2023robust} propose to jointly perform structure learning and OT alignment through minimizing multi-view GW distance matrices between two attributed graphs. These methods have primarily used OT to define a learning objective, which involves bi-level optimization for model training. To further enhance the scalability of OT modeling for entity alignment, \citet{mao2022lightea} propose to make the similarity matrix sparse by dropping its entries close to zero. However, this sparse OT modeling potentially violates the constraints of the OT objective, failing to guarantee one-to-one correspondences across two KGs. In our work, we focus on tackling the scarcity of seed alignments via iterative pseudo-labeling; we seek to find more accurate one-to-one alignment configurations between entities via OT modeling, thus eliminating conflicted misalignments at each pseudo-labeling iteration and mitigating confirmation bias. 


\section{Conclusion and Future Work}
\label{section:conclusion}

\rv{We have investigated the problem of confirmation bias for pseudo-labeling-based entity alignment, which has been largely overlooked in the literature. Through an in-depth analysis, we have revealed the underlying causes of confirmation bias and proposed UPL-EA, a novel unified pseudo-labeling framework for entity alignment. UPL-EA systematically addresses confirmation bias through two key innovations: OT-based pseudo-labeling and \rvtwo{parallel pseudo-label ensembling}. OT-based pseudo-labeling utilizes a discrete OT formulation to more accurately infer pseudo-labeled alignments that satisfy one-to-one correspondences}, thus mitigating \rvtwo{conflicted misalignments}. \rvtwo{Parallel pseudo-label ensembling combines the predictions of pseudo-labeled alignments from multiple \rvtwo{OT-based models independently trained in parallel} to reduce variability in pseudo-label selection}, thus alleviating the propagation of \rvtwo{one-to-one misalignments} into subsequent model training. \rv{Our extensive experimental evaluation and analysis demonstrate that UPL-EA outperforms state-of-the-art baselines across various types of benchmark datasets}. The competitive performance of UPL-EA validates its superiority in addressing confirmation bias and \rv{its utility as a general pseudo-labeling framework to improve entity alignment performance. Future research will include a theoretical investigation to rigorously assess the effectiveness of pseudo-labeling ensembling within UPL-EA. Additionally, we will explore extending our OT formulation to incorporate different sources of information for multi-modal entity alignment.}

\section*{Compliance with Ethical Standards}
The authors have no conflict of interests or competing interests to declare that are relevant to the content of this article.









\bibliography{ea-bibs}

\end{document}